\documentclass[10pt, twoside]{article}
\usepackage[margin=1in]{geometry}

\usepackage{url}            % simple URL typesetting
\usepackage{subfig}

\usepackage{amsfonts,amsmath,amssymb,amsthm,bm,bbm}

\usepackage{algorithm}
\usepackage{algorithmic}
\usepackage{paralist}

\usepackage{dsfont,array}

\usepackage{newclude}

\usepackage{appendix}

\usepackage[colorlinks,
linkcolor=blue,
citecolor=blue,
urlcolor=magenta]{hyperref}

\usepackage{multirow}
\usepackage{multicol}
\usepackage{xspace}
\usepackage{mathtools}
\usepackage{placeins}
\usepackage{authblk}

\newtheorem{definition}{Definition}
\newtheorem{theorem}{Theorem}
\newtheorem{lemma}{Lemma}

\newtheorem{remark}{Remark}

\DeclareMathOperator*{\argmax}{arg\,max}
\newcommand\numberthis{\addtocounter{equation}{1}\tag{\theequation}}

\DeclarePairedDelimiterX\Basics[1](){ #1}

\newcounter{const-no}

\def\EE{{\mathbb{E}}}\def\PP{{\mathbb{P}}}

\DeclareMathOperator*{\argmin}{\arg\!\min}

\date{\today}
\title{Identifying Best Interventions through Online Importance Sampling}

\author[1]{Rajat Sen}
\author[2]{Karthikeyan Shanmugam}
\author[1]{Alexandros G. Dimakis}
\author[1]{Sanjay Shakkottai}
\affil[1]{The University of Texas at Austin}
\affil[2]{IBM Thomas J. Watson Research Center}

\begin{document}
\maketitle

\begin{abstract} 
  Motivated by applications in computational advertising and systems
  biology, we consider the problem of identifying the best out of
  several possible \textit{soft interventions} at a \textit{source}
  node $V$ in an acyclic causal directed graph, to maximize the
  expected value of a \textit{target} node $Y$ (located downstream of
  $V$). Our setting imposes a fixed total budget for sampling under various
  interventions, along with cost constraints on different types
  of interventions. We pose this as a best arm identification bandit
  problem with $K$ arms where each arm is a soft intervention at
  $V,$ and leverage the information leakage among the arms to provide
  the first gap dependent error and simple regret bounds for this
  problem. Our results are a significant improvement over the
  traditional best arm identification results. We empirically show
  that our algorithms outperform the state of the art in the Flow
  Cytometry data-set, and also apply our algorithm for model
  interpretation of the Inception-v3 deep net that classifies images.

%  We propose an efficient \textit{successive rejects} style algorithm
%  that uses \textit{importance sampling} to adaptively sample using
%  different interventions and leverage information leakage to choose
%  the best. We provide the first \textit{gap dependent} (gaps between
%  the various expected target values) simple regret and best arm
%  mis-identification error bounds for this problem. This generalizes
%  the prior bounds available for the simpler case of no information
%  leakage. In the case of no leakage, the number of samples required
%  for identification is  (upto polylog factors)
%  $\tilde{O} (\max_i \frac{i}{\Delta_i^2})$ where $\Delta_i$ is the
%  gap between the optimal and the $i$-th best arm. %The term inside the
%  %max operator can be interpreted to be the number of samples required
%  %to eliminate the $i$-th best arm out of $i$ best arms. 
% We generalize the previous result for the causal setting and show that
%  $\tilde{O}(\max_i \frac{\sigma_i^2}{\Delta_i^2})$ is sufficient
%  where $\sigma_i^2$ is the \textit{effective variance} of an importance
%  sampling estimator that eliminates the $i$-th best arm out of a set of arms with gaps roughly at most twice as big as $\Delta_i$. We also show that
%  $\sigma_i^2 << i$ in many cases. Our result also recovers (up to
%  constants) prior gap independent bounds for this setting. We
%  demonstrate that our algorithm empirically outperforms the state of
%  the art, through synthetic experiments.
\end{abstract}

\section{Introduction}
\label{sec:intro}

Causal graphs \cite{pearl2009causality} are useful for representing
causal relationships among interacting variables in large systems
\cite{bottou2013counterfactual}. Over the last few decades, causal
models have found use in computational
advertising~\cite{bottou2013counterfactual}, biological
systems~\cite{meinshausen2016methods},
sociology~\cite{blalock1985causal},
agriculture~\cite{splawa1990application} and epidemiology
\cite{joffe2012causal}.
There are two important questions commonly studied with causal graphs: {\em
(i)} How to learn a directed causal graph that encodes the pattern of
interaction among components in a system (\textit{casual structure
learning})?~\cite{pearl2009causality}~, and {\em (ii)} Using
previously acquired (partial) knowledge about the causal graph
structure, how to estimate and/or to optimize the effect of a new
\textit{intervention} on other variables (\textit{optimization})~\cite{bottou2013counterfactual,joffe2012causal,kemmeren2014large,bonneau2007predictive,krouk2013gene}?
Here, an \textit{intervention} is a forcible change to the value of a
variable in a system. The change either alters the relationship
between the parental causes and the variable, or decouples it from the
parental causes entirely. Our focus is on optimizing over a given set of interventions.

\begin{figure}
	\centering
	\includegraphics[width=0.6\linewidth]{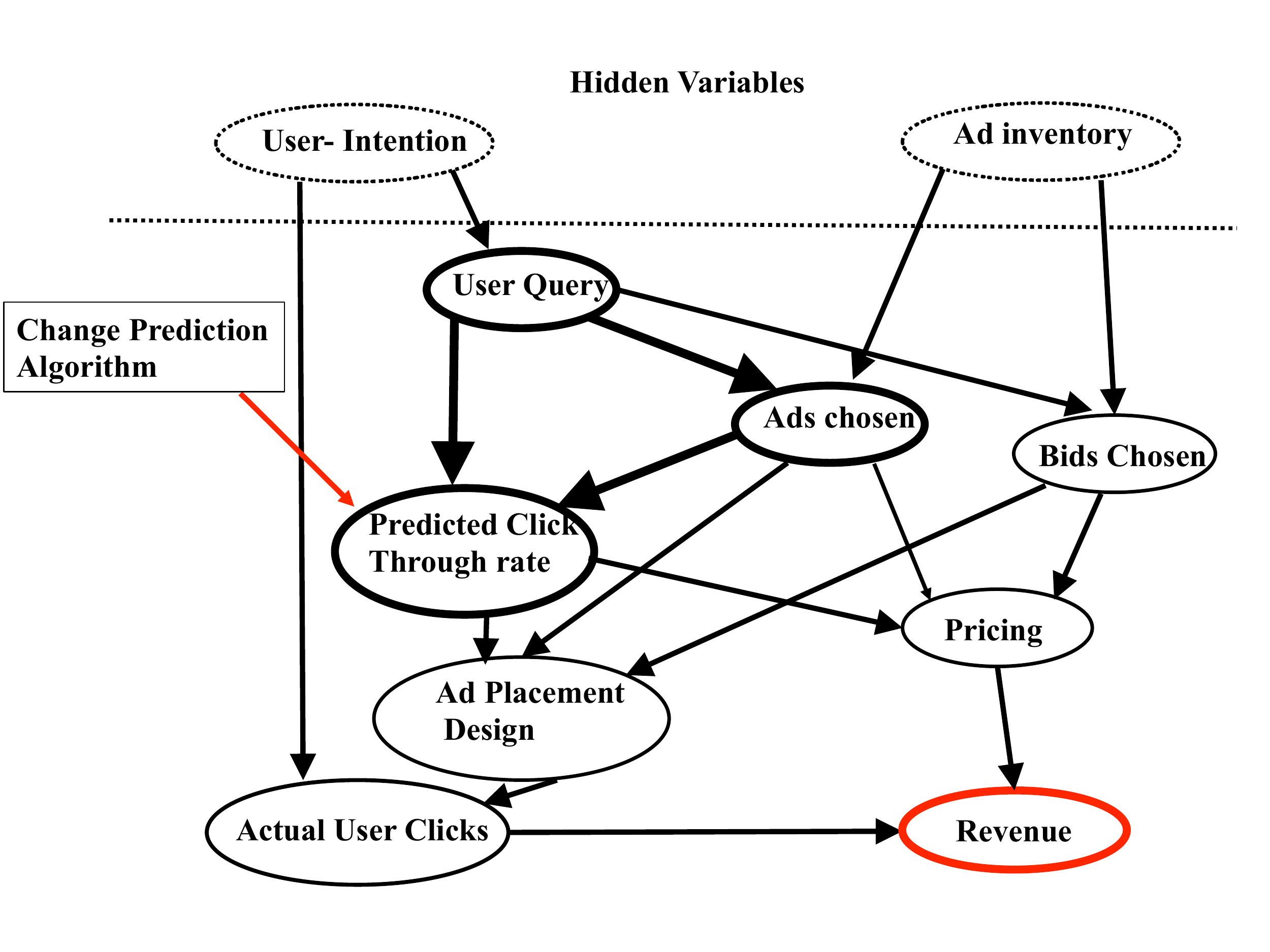}
	\caption{ \small Computational advertising example borrowed from
		\cite{bottou2013counterfactual}. Various observable and hidden variables are shown. The topology of the causal graph is known;
		however the strengths of most interactions are unknown. A click-rate
		scoring algorithm predicts future user click through rates from users' search
		queries and the set of ads relevant to the user query chosen
		from an ad inventory. The algorithm's output determines the ads
		displayed (as well as the display style), and through a complex causal
		graph, finally determines actual revenue. The part of the network in
		bold -- distribution of user queries and matching ad keywords is known
		(including strengths), and the input output characteristics of several
		candidate (randomized) click-rate scoring algorithms are known. The
		objective is to choose the best algorithm that maximizes the revenue
		(target in bold red).\normalsize} \label{fig:Bottou}
\end{figure}

An illustrative example includes online advertising
\cite{bottou2013counterfactual}, where there is a collection of
click-through rate scoring algorithms that provide an estimate of the
probability that an user clicks on an ad displayed at a specific
position. The interventions occur through the choice of click-through
rate scoring algorithm; the algorithm choice directly impacts ad
placement and pricing, and through a complex network of interactions,
affects the revenue generated through advertisements. The revenue
%(downstream node in a causal
%graph)
 is used to determine the best scoring algorithm (optimize for the
best intervention); see Figure~\ref{fig:Bottou}. Another example is in
biological gene-regulatory networks \cite{bonneau2007predictive},
where a large number of genomes interact amongst each other and also
interact with environmental factors. The objective here is to
understand the best perturbation of some genomes in terms of its
effect on the expression of another subset of genomes (target) in
cellular systems. %(the downstream node).

This paper focuses on the following setting: We are given apriori
knowledge about the structure and strength of interactions over a
small part of the causal graph.  In addition, there is freedom to
intervene (from a set of allowable interventions) at a certain node in
the known part of the graph, and collect data under the chosen
intervention; further we can alter the interventions over time and
observe the corresponding effects. Given a set of potential
interventions to optimize over, the key question of interest is: How
to choose the best sequence of $T$ allowable interventions in order to
discover which intervention maximizes the expectation of a
downstream target node?

Determining the best intervention in the above setting can be cast as
a best arm identification bandit problem, as noted in
\cite{lattimore2016causal}. The possible interventions to optimize over are the arms of the bandit, while the sample value of the \textit{target} node under an intervention is the \textit{reward}.

More formally, suppose that $V$ is a node in a causal graph
$\mathcal{G}(\mathcal{V}, \mathcal{E})$ (as shown in
Fig.~\ref{fig:illustrate}), with the parents of $V$ denoted by
$pa(V).$ In Fig.~\ref{fig:Bottou}, $V$ corresponds to the
\textit{click-through rate} and its parents are \textit{user-query}
and \textit{ads-chosen}. This essentially means that $V$ is causally
determined by a function of $pa(V)$ and some exogenous random
noise. This dependence is characterized by the conditional
$\mathrm{P}(V | pa(V))$\footnote{Formally if node $V$ has parents
  $V_1,V_2$, then this distribution is the conditional
  $\mathrm{P}(V=v | V_1 = v_1, V_2 = v_2)$ for all $v,v_1,v_2$.}. Then
a (soft) intervention mathematically corresponds to changing this
conditional probability distribution i.e. probabilistically forcing
$V$ to certain states given its parents.
% \footnote{These changes through
%   the conditional distributions are also called \textit{soft
%     interventions}~\cite{pearl2009causality}.Note that our algorithm
%   is also applicable to the more general case when the set of
%   interventions to optimize over is much larger than the set of
%   interventions available to sample from. This 'transfer of knowledge'
%   to an unseen but informationally related situation is also an
%   important motivation. Formally, it is possible to choose the best
%   soft intervention in $\{0,1 \cdots , K-1 \}$ by only observing
%   samples under interventions in the set $\mathcal{S}$, where
%   $\mathcal{S} \subset [K]$.}. 
In the computational advertising
example, the interventions correspond to changing the click through
rate scoring algorithm i.e
$\mathrm{P}(click~through~rate \vert ads~chosen,user~query)$, whose
input-output characteristics are well-studied. Further, suppose that
the effect of an intervention is observed at a node $Y$ which is
downstream of $V$ in the topological order (w.r.t $\mathcal{G}$)
-refer to Fig.~\ref{fig:illustrate}.  Then, our key question is stated
as follows: \textit{Given a collection of interventions
  $\{\mathrm{P}_0 (V| pa(V)),\ldots \mathrm{P}_{K-1}(V | pa(V))\},$
  find the best intervention among these $K$ possibilites that
  maximizes $E[Y]$ under a fixed budget of $T$ (intervention,
  observation) pairs.}

%The key difference from the classical best arm setting is that there is \textit{information leakage} among the arms through the shared causal graph. In the next section we list our main contributions, where we elaborate on the advantages of leveraging this information leakage. 

%A key property we use is that these interventions are correlated,
%i.e. there is information leakage that informs us about how good one
%arm is by observing the performance of another. Classical best arm
%identification~\cite{audibert2010best} does not model such information
%leakage among arms; our goal is to extend the theory for this regime
%and quantify the benefits of information leakage. Information leakage
%effects are natural for various problems. For example, Bottou et
%al.~\cite{bottou2013counterfactual} show that the performance of
%untested ad scoring algorithms can be predicted from different scoring
%models.  Recent work \cite{lattimore2016causal} has leveraged this
%information leakage and developed algorithms for approximating the
%best possible sequence of interventions for discovering the optimal
%policy.

\begin{figure}
	\centering
	\includegraphics[width=0.5\linewidth]{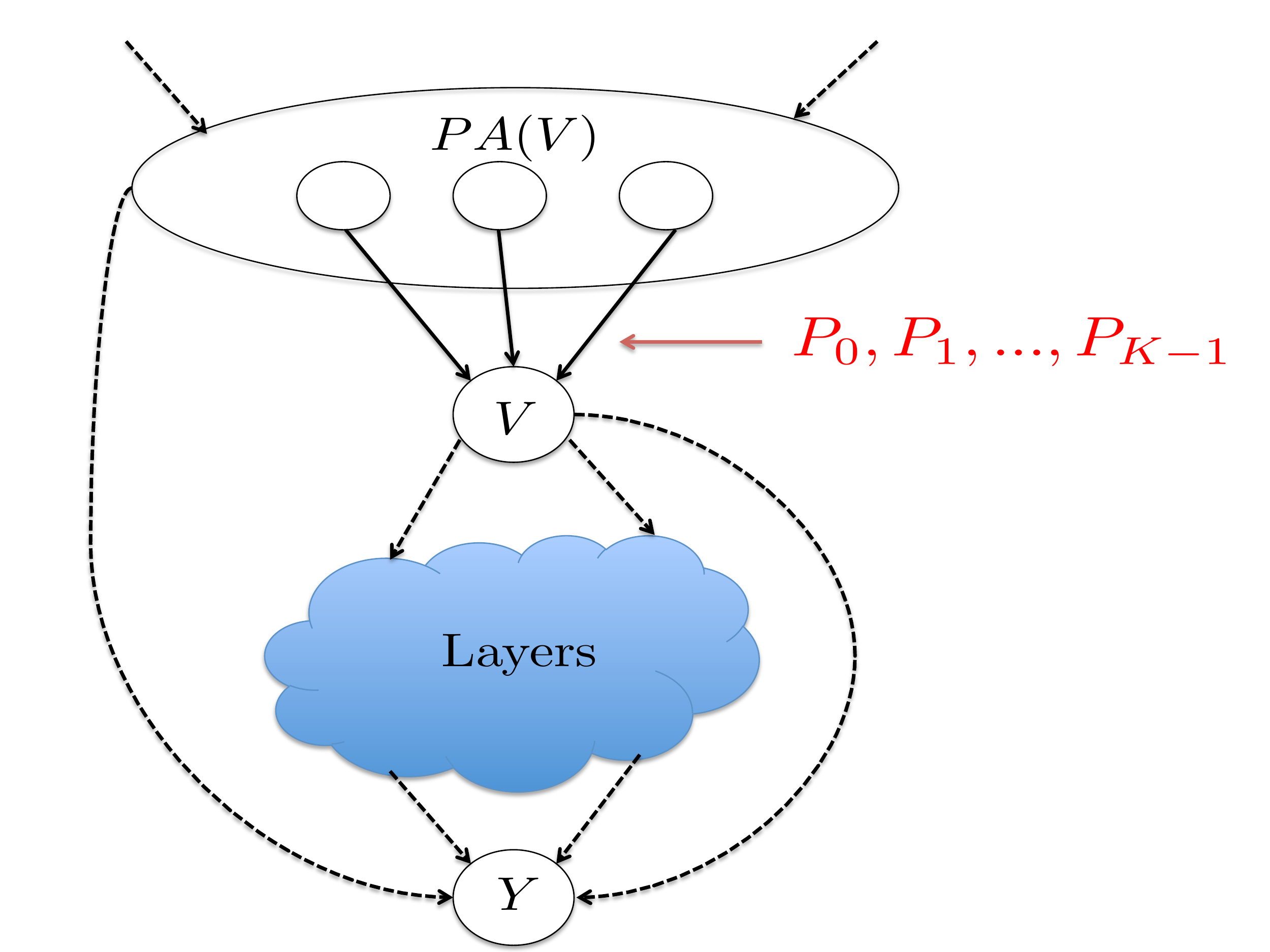}
	\caption{\small Illustration of our setting. The soft interventions modify $\mathrm{P}(V | pa(V)$. The various conditionals and the marginal of $pa(V)$ is assumed to be known from prior data. The \textit{target} variable $Y$ lies further downstream in the unknown portions of the causal graph. \normalsize} \label{fig:illustrate}
\end{figure}

%Our main result in this paper is a \textit{successive rejects bandit
%algorithm} for selecting the sequence of interventions, to determine
%the best one that has the maximum desired effect on a downstream
%node. Our algorithm uses past observations along with information
%leakage across interventions to sequentially eliminate sub-optimal
%interventions. We provide the first \textit{gap-dependent} simple
%regret and error bounds for this problem. Further, our solution
%explicitly accounts for possible cost constraints on using
%interventions. As a special case, we recover the gap-independent
%result in \cite{lattimore2016causal} (orderwise).

 \subsection{Main Contributions}
 
 {\noindent \bf (Successive Rejects Algorithm)} We provide an
 efficient \textit{successive rejects} multi-phase algorithm. The
 algorithm uses \textit{clipped importance sampling}. The clipper
 level is set \textit{adaptively} in each phase in order to trade-off
 \textit{bias and variance}. Our procedure yields a major improvement
 over the algorithm in~\cite{lattimore2016causal} (both in theoretical
 guarantees and in practice), which sets the clippers and allocates
 samples in a static manner.

 {\noindent \bf (Gap Dependent Error Guarantees under Budget
   Constraints)} In the classic best arm identification
 problem~\cite{audibert2010best}, Audibert et al. derive \textit{gap}
 dependent bounds on the probability of error given a fixed sample
 budget.  Specifically, let $\Delta_{(i)}$ be the $i$-th largest gap
 (difference) in the expected reward from that of the best arm
 (e.g. $\Delta_{(1)}$ is the difference between the best arm expected
 reward and the second best reward). Then, it has been shown in
 \cite{audibert2010best} that the number of samples needed scales as
 (upto poly log factors) $\max_i (i/\Delta_{(i)}^2).$

 % Audibert et al. derive \textit{gap} dependent bounds for the classic
 % best arm identification problem~\cite{audibert2010best}. It has been
 % shown that the number of samples needed is (upto poly log factors)
 % $\max_i (i/\Delta_{(i)}^2)$ where $\Delta_{(i)}$ is the $i$-th
 % smallest \textit{gap} from the optimal (in terms expected reward of
 % the arms).

 In our setting, a fundamental difference from the classical best arm
 setting~\cite{audibert2010best} is the \textit{information leakage}
 across the arms, i.e, samples from one arm can inform us about the
 expected value of other arms because of the shared causal graph. We
 show that this information leakage yields significant improvement
 both in theory and practice. We derive the first \textit{gap
   dependent} (gaps between the expected reward at the target under
 different interventions) bounds on the
 probability of error in terms of the number of samples $T$, cost
 budget $B$ on the relative fraction of times various arms are sampled
 and the divergences between soft intervention distributions of the
 arms.
 
 In our result (upto poly log factors) the factor $i$ is replaced by
 the 'effective variance' of the estimator for arm $(i)$, i.e. we
 obtain (with informal notation) $\max_i
 \sigma_i^2/\Delta_{(i)}^2$. $\sigma_i$ can be much smaller than
 $\sqrt{i}$ (the corresponding term in the results
 of~\cite{audibert2010best}). Our theoretical guarantees quantify the
 improvement obtained by leveraging information leakage, which has
 been empirically observed in~\cite{bottou2013counterfactual}. We
 discuss in more detail in Sections~\ref{sec:guarantee}
 and~\ref{sec:appen_compare} (in the appendix), about how these
 guarantees can be exponentially better than the classical ones.  We
 derive simple regret (refer to Section~\ref{sec:def}) bounds
 analogous to the gap dependent error bounds.

% 
% A fundamental difference from the classical best arm setting~\cite{audibert2010best} is the \textit{information leakage} among the arms i.e sampling from one arm can inform us about the expected value of other arms because of  the shared causal graph. We show that leveraging this information leakage yields significant improvement both in theory and practice.
% 
% Our result generalizes that of \cite{audibert2010best} (the classic best arm
% identification problem)
% which derives that the number of samples needed is (upto poly
% log factors) $\max_i (i/\Delta_{(i)}^2)$ where $\Delta_{(i)}$ is the
% $i$-th smallest gap from the optimal . We obtain a similar result (upto poly log factors) where the the
%factor $i$ is replaced by the 'effective variance' of the estimator
%for arm $(i)$, i.e. we obtain (with informal notation) $\max_i
%\sigma_i^2/\Delta_{(i)}^2$. $\sigma_i$ can be much smaller than $\sqrt{i}$ (the corresponding term in the results of~\cite{audibert2010best}). Our theoretical guarantees quantify the improvement obtained by leveraging information leakage, which have been empirically observed in~\cite{bottou2013counterfactual}. We discuss in more detail in Sections~\ref{sec:guarantee} and~\ref{sec:detcontri} (in the appendix), about how these guarantees can be exponentially better than the classical ones.
%We derive simple regret (refer to Section~\ref{sec:def}) bounds analogous to
% the gap dependent bounds on the probability of misidentification.
 
{\noindent \bf (Novel $f$-divergence measure for analyzing Importance Sampling)} We provide a novel analysis of \textit{clipped importance sampling} estimators, where pairwise $f$-divergences between the distributions
 $\{\mathrm{P}_{k}(V \vert pa(V))\}$, for a carefully chosen function $f(.)$
 (see Section~\ref{sec:divergence}) act as the `effective variance'
 term in the analysis for the estimators (similar to Bernstein's bound~\cite{bennett1962probability}).

{\noindent \bf (Extensive Empirical Validation)} We demonstrate that our algorithm outperforms the prior works~\cite{lattimore2016causal,audibert2010best} on the Flow Cytometry data-set~\cite{sachs2005causal} (in Section~\ref{sec:cytometry}). We exhibit an innovative application of our algorithm for \textit{model interpretation} of the Inception Deep Network~\cite{szegedy2015going} for image classification (refer to Section~\ref{sec:inception}).

\begin{remark}
The techniques in this paper can be directly applied to more general settings like (i) the intervention source ($V$) can be a collection of nodes ($\mathcal{V}$) and the changes affect the distribution $P(\mathcal{V} \vert pa(\mathcal{V}))$, where  $pa(\mathcal{V})$ is the union of all the parents; (ii) the importance sampling can be applied at a directed cut separating the sources and the targets, provided the effect of the interventions, on the nodes forming the cut can be estimated. Moreover, our techniques can be applied without the complete knowledge of the source distributions. We explain the variations in more detail in Section~\ref{sec:variation} in the appendix.
\end{remark}

 \subsection{Related Work}
 The problem lies at the intersection of causal inference and best arm
 identification in bandits. There have been many studies on the
 classical \textit{best arm identification} in the bandit literature,
 both in the fixed confidence
 regime~\cite{kaufmann2015complexity,gabillon2012best} and in the
 fixed budget
 setting~\cite{audibert2010best,chen2015optimal,jamieson2014lil,carpentier2016tight}. It
 was shown recently in \cite{carpentier2016tight} that the results of
 \cite{audibert2010best} are optimal.  The key difference from our
 work is that, in these models, there is no information leakage among
 the arms.

 There has been a lot of work
 \cite{mooij2016distinguishing,Hyttinen2013,Eberhardt2008X,Hauser2012b,Spirtes2001,Pearl2009,loh2014high,shanmugam2015learning}
 on \textit{learning casual models} from data and/or experiments and
 using it to \textit{estimate causal strength} questions of the
 \textit{counterfactual} nature. One notable work that partially
 inspired our work is \cite{bottou2013counterfactual} where the causal
 graph underlying a computational advertising system (like in Bing,
 Google etc.) is known and the primary interest is to find out how a
 change in the system would affect some other variable.
  
At the intersection of \textit{causality and bandits}, \cite{lattimore2016causal} is perhaps most
relevant to our setting. It studies the problem of identifying the best hard interventions on multiple variables (among many), provided the distribution of the parents of the target is known under those interventions. Simple regret bound of order $O(1/\sqrt{T})$ was derived. We assume soft interventions that affect the mechanism between
a 'source' node and its parents, far away from the target (similar to
the case of computational advertising considered in
\cite{bottou2013counterfactual}). Further, we derive the first
\textit{gap dependent} bounds (that can be exponentially small in $T$),
generalizing the results of \cite{audibert2010best}. Our formulation can handle general budget constraints on the bandit arms and also recover the problem independent bounds of~\cite{lattimore2016causal} (orderwise). Budget constraints in bandit settings have been explored before in~\cite{abernethy2015low,slivkins2013dynamic}.

In the context of machine learning, \textit{importance sampling} has been mostly used to recondition input data to adhere to conditions imposed by learning algorithms~\cite{sugiyama2007covariate,li2011unbiased,zhang2014stochastic}.

%\subsection{Related Work}
%\input{Rwork}
%
%\subsection{Contributions}

\section{Problem Setting}
\label{sec:probdef}
A causal graph $\mathcal{G}(\mathcal{V},\mathcal{E})$ specifies causal relationships among the random variables representing the vertices of the graph $\mathcal{V}$. The relationships are specified by the directed edges $\mathcal{E}$; an edge $V_i \rightarrow V_j$ implies that $V_i \in \mathcal{V}$ is a direct parental cause for the effect $V_j\in \mathcal{V}$. With some abuse of notation, we will denote the random variable associated with a node $V \in \mathcal{V}$ by $V$ itself. We will denote the parents of a node $V$ by $pa(V)$. The causal dependence implies that $V = f_V ( {U \in pa(V)}, \epsilon_V)$, where $\epsilon_V$ is an independent exogenous noise variable. One does not get to measure the functions $f_V$ in practice. The noise variable and the above functional dependence induce a conditional probability distribution $\mathrm{P}(V | pa(V) )$. Further, the joint distribution of $\{V\}_{V \in \mathcal{V}}$ decomposes into product of conditional distributions according to $\mathcal{G}$ viewed as a Bayesian Network, i.e.  $\mathrm{P}( \{V\}_{V \in \mathcal{V}}) = \prod \limits_{V \in \mathcal{V}} \mathrm{P}(V | pa(V) )$. 

\textit{Interventions} in a causal setting can be categorized into two kinds:

\begin{enumerate}
	\item \textit{Soft Interventions}: At node $V$, the conditional distribution relating $pa(V)$ and $V$ is changed to $\tilde{\mathrm{P}} \left(V | pa(V) \right)$. 
	\item \textit{Hard Interventions:} We force the node $V$ to take a specific value $x$. The conditional distribution $\tilde{\mathrm{P}} \left(V | pa(V) \right)$ is set to a point mass function $\mathbf{1}_{V =x}$.
\end{enumerate}

In this work, we consider the problem of identifying the best soft intervention, i.e. the one that maximizes the expected value of a certain target variable. The problem setting is best illustrated in Figure~\ref{fig:illustrate}.
Consider a causal graph $\mathcal{G}(\mathcal{V},\mathcal{E})$ that specifies directed causal relationships between the variables $\mathcal{V}$. %We provide a short introduction to causal graphs in Section~\ref{sec:causalintro} (in the appendix) for the sake of completeness.  
Let $Y$ be a \textit{target} random variable which is downstream in the graph $\mathcal{G}$; the expected value of this target variable is the quantity of interest. Consider another random variable $V$ along with its parents $pa(V)$. We assume that there are $K$ possible soft interventions. Each soft intervention is a distinct conditional distribution that dictates the relationship $pa(V) \rightarrow V$.  During a soft intervention $k \in [K]$ ($[K] = \{0,1,...,K-1 \}$), the conditional distribution of $V$ given its parents is set to $\mathrm{P}_k(V \vert pa(V))$ and all other relationships in the causal graph are unchanged. 

It is assumed that the conditional distributions $\mathrm{P}_k(V \vert pa(V))$ and marginals for $pa(V)$ for $k \in [K]$ are known from past experiments or existing domain knowledge. We only observe samples of $Y, V$ and $pa(V)$, while the rest of the variables in the causal graph may be unobserved under different interventions. For simplicity we assume that the variables $V,pa(V)$ are discrete while the target variable $Y$ may be continuous/discrete and has bounded support in $[0,1]$. Further, we assume that the various conditionals, i.e. $\mathrm{P}_k \left(V \vert pa(V \right))$ are \textit{absolutely continuous} with respect to each other. In the case of discrete distributions, the non-zero supports of these distributions are identical. However, our algorithm can be easily generalized for continuous distributions on $V$ and $pa(V)$ (as in our experiments in Section~\ref{sec:cytometry}). In this setting, we are interested in the following natural questions:{ \it Which of the $K$ soft interventions yield the highest expected value of the target $(\mathbb{E}[Y])$ and what is the misidentification error that can be achieved with a finite total budget $T$ for samples ?}

\begin{remark} Although we may know apriori the joint distribution of $pa(V)$ and $V$ under different interventions, how the change affects another variable $Y$ in the causal graph is unknown and must be learnt from samples. The task is to \textbf{transfer prior knowledge} to identify the best intervention.
	\end{remark}

%In the following subsections, we will establish that this question can be naturally posed in the setting of best arm identification in bandits~\cite{audibert2010best, gabillon2012best, carpentier2016tight}. 

{\noindent \bf Bandit Setting: } The $K$ different soft interventions can be thought of as the $K$ \textit{arms} of a bandit problem. Let the reward of arm $k$ be denoted by: $\mu_k = \EE_{k} \left[ Y \right]$, where $\EE_{k} \left[ Y \right]$ is the expected value of $Y$ under the soft intervention when the conditional distribution of $V$ given its parents $pa(V)$ is set to $\mathrm{P}_k(V \vert pa(V))$ (soft intervention $k$) while keeping all other things in $\mathcal{G}$ unchanged. We assume that there is only one best arm. Let $k^*$ be the arm that yields the highest expected reward and $\mu^*$ be the value of the corresponding expected reward i.e. $k^* = \argmax_{k} \mu_k$ and $\mu^* = \mu_{k^*}.$ Let the optimality gap of the $k^{th}$ arm be defined as $\Delta_{k} = \mu^* - \mu_k$. We shall see that the these gaps $\{\Delta_{k} \}_{k=0}^{K-1}$ and the relationship between distributions $\{ \mathrm{P}_k(V \vert pa(V))\}_{k = 0}^{K-1}$ are important parameters in the problem. Let the minimum gap be $\Delta = \min_{k \neq k^*} \Delta_k$.

{\noindent \bf Fixed Budget for Samples: } In this paper, we work under the fixed budget setting of best arm identification~\cite{audibert2010best}. Let $T_k$ be the number of times the $k^{th}$ intervention is used to obtain samples. We require that $ \sum_{k=0}^K T_k = T$. Let $\nu_k = \frac{T_k}{T}$ be the fraction of times the $k^{th}$ intervention is played.

{\noindent \bf Additional Cost Budget on Interventions: }
In the context of causal discovery, some \textit{interventions} require a lot more resources or experimental effort than the others. We find such examples in the context of online advertisement design~\cite{bottou2013counterfactual}.  Therefore, we introduce two variants of an additional cost constraint that influences the choice of interventions.  {\it (i) Difficult arm budget} (\textbf{S1}): Some arms are deemed to be \textit{difficult}. Let $\mathcal{B} \subset [K]$ be the set of \textit{difficult} arms. We require that the total fraction of times the difficult arms are played does not exceed $B$ i.e. $\sum_{k \in \mathcal{B}} \nu_k \leq B.$
{\it (ii) Cost Budget} (\textbf{S2}): This is the most general budget setting that captures the variable costs of sampling each arm~\cite{slivkins2013dynamic}. We assume that there is a cost $c_k$ associated with sampling arm $k$. It is required that the average cost of sampling does not exceed a cost budget $B$ .ie. $\sum_{k = 0}^{K-1} c_k \nu_k \leq B.$ 
$\mathbf{c} = [c_1,..,c_k]$ along with the total budget $T$ completely defines this budget setting. It should be noted that \textbf{S1} is a special case of \textbf{S2}.

\noindent We note that unless otherwise stated, we work with the most general setting in \textbf{S2}. We state some of our results in the setting \textbf{S1} for clearer exposition.

{\noindent \bf Objectives:} There are two main quantities of interest:

(\textit{Probability of Error}): This is the probability of failing to identify the best soft intervention (arm). Let $\hat{k}(T,B)$ be the arm that is predicted to be the best arm at the end of the experiment. Then the probability of error $e(T,B)$~\cite{audibert2010best, carpentier2016tight} is given by, 

$e(T,B) = \PP \left( \hat{k}(T,B) \neq k^*\right)$

(\textit{Simple Regret}): Another important quantity that has been analyzed in the best arm identification setting is the simple regret~\cite{lattimore2016causal}. The simple regret is given by $r(T,B) = \sum _{k \neq k^*}\Delta_k\PP \left( \hat{k}(T,B) = k\right)$.

%\noindent In this paper we will provide efficient algorithms that have provable guarantees on both $r(T,B)$ and $e(T,B)$. We will analyze our algorithms in a problem dependent setting where the guarantees are dependent of the problem parameters $\{\Delta_{k} \}_{k=0}^{K-1}$ and $\{ \mathrm{P}_k(V \vert pa(V))\}_{k = 0}^{K-1}$. The guarantee on $r(T,B)$ also generalizes to the problem independent setting where the gaps $\{\Delta_{k} \}_{k=0}^{K-1}$ can be arbitrarily close to $0$. 

\section{Our Main Results}
\label{sec:results}
In this section we provide our main theoretical contributions. In
Section~\ref{sec:algo}, we provide a successive rejects style
algorithm that leverages the information leakage between the arms via
importance sampling. Then, we provide theoretical guarantees on the
probability of mis-identification $(e(T,B))$ and simple regret
$(r(T,B))$ for our algorithm in Section~\ref{sec:guarantee}. In order
to explain our algorithm and our results formally, we first describe
several key ideas in our algorithm and introduce important definitions in
Section~\ref{sec:def}.

\subsection{Definitions}
\label{sec:def}
% The algorithm starts by having all the $K$ arms under consideration
% and then proceeds in phases. At the beginning of each phase, a
% budget for samples is allocated to each arm. Thereafter, samples are
% collected according to the allocated budgets, and using these
% samples, importance sampling based estimators for the means of all
% the \textit{remaining} arms are formed. Then at the end of each
% phase, one or more arms are rejected based on some phase specific
% criterion on the estimated values. The structure of the algorithm
% and the estimation technique for the means of the arms, both rely on
% a specific $f$-divergence between the arm distributions.

\textbf{Quantifying Information Leakage:} %We observe that
% \[\mu_k = \mathbb{E}_k[Y] = \mathbb{E}_{k'}\left[Y
%     \frac{\mathrm{P}_k\left(V|pa(V)
%       \right)}{\mathrm{P}_{k'}\left(V|pa(V)\right)}\right]. \] By
% weighting samples with the correct ratio of conditional probabilities
% $\mathrm{P}_k(\cdot)$ and $\mathrm{P}_{k'}(\cdot)$, it is possible to
% use samples under intervention $k'$ to estimate the mean under
% intervention $k$. However, the variance of this estimator depends on
% the ratio of $\mathrm{P}_k(\cdot)$ to $\mathrm{P}_{k}'(\cdot)$.
% We use a specific $f$-divergence measure between
% $\mathrm{P}_k(\cdot)$ and $\mathrm{P}_{k}'(\cdot)$ to quantify
% the variance. This $f$-divergence measure can be calculated
% analytically or estimated directly from empirical data (without the
% knowledge of the full distributions) using techniques like that
% of~\cite{perez2008kullback}. 
Our setting is one in which there is information leakage among the $K$
arms of the bandit. Recall that each arm corresponds to a different
conditional distribution imposed on a node $V$ given its parents
$pa(V)$, while the rest of the relationships in the causal graph
$\mathcal{G}$ remain unchanged. Since the different candidate
conditional distributions $\mathrm{P}_{k}(V \vert pa(V))$ are known
from prior knowledge (and are absolutely continuous with respect to
each other), it is possible to utilize samples obtained under an arm
$j$ to obtain an estimate for the expectation under some other arm $i$
(i.e $\EE_{i} \left[ Y\right]$). A popular method for utilizing this
information leakage among different distributions is through
\textit{importance sampling}, which has been used in counterfactual
analysis in similar causal settings~\cite{lattimore2016causal,
  bottou2013counterfactual}.

{\bf \noindent Importance Sampling:} 
% Now we introduce the concept of importance sampling which is one of
% the key tools we use to leverage the information leakage between the
% candidate arms.
Suppose we get samples from arm $j \in [K]$ and we are interested in
estimating $\EE_i[Y]$. In this context it helpful to express
$\EE_i[Y]$ in the following manner:
\begin{align}
\label{eq:basic}
\EE_{i}\left[Y\right] = \EE_{j}\left[Y\frac{\mathrm{P}_i(V \vert pa(V))}{\mathrm{P}_j(V \vert pa(V))}\right]
\end{align}
(\ref{eq:basic}) is trivially true because the only change to the
joint distribution of all the variables in the causal graph
$\mathcal{G}$ under arm $i$ and $j$ is at the factor $\mathrm{P}(V
\vert pa(V))$. Suppose we observe $t$ samples of $\{Y, V, pa(V) \}$
from the arm $j$, denoted by
$\{Y_j(s),V_j(s),pa(V)_j(s)\}_{s=1}^{t}$. Here $X_j(s)$ denotes the
sample from random variable $X$ at time step $s$, while the sub-script
$j$ just denotes that the samples are collected under arm $j$. Under
the observation of  Equation (\ref{eq:basic}), one might assume that
the naive estimator,  
\begin{align}
\label{eq:naive}
\hat{Y}'_i(j) = \frac{1}{t}\sum_{s = 1}^{t} Y_j(s)\frac{\mathrm{P}_i(V_j(s) \vert pa(V)_j(s))}{\mathrm{P}_j(V_j(s) \vert pa(V)_j(s))}
\end{align} 
provides a good estimate for $\mu_i = \EE_{i} \left[
  Y\right]$. However, the confidence guarantees on such an estimate
can be arbitrarily bad even if $Y$ is bounded. This is because the
factor $\mathrm{P}_i(V \vert pa(V))/\mathrm{P}_j(V \vert pa(V))$ can
be very large for several instances of $(V, pa(V))$. Therefore, usual
measure concentrations (e.g. the Azuma-Hoeffding inequality) would not
yield good confidence intervals. This has been noted
in~\cite{lattimore2016causal} in a similar setting, where a
\textit{static} clipper has been applied to the weighted samples to
control the variance. However, a static clipper introduces a fixed
bias, and thus it is not suitable for obtaining gap dependent simple regret
bounds. Instead, in our algorithm we a multi-phase approach and use
dynamic clipping to adaptively control the bias vs. variance trade-off
in a phase dependent manner, which leads to significantly better gap
dependent bounds\footnote{We note that the authors in
  \cite{lattimore2016causal} discuss the possibility of a multi-phase
  approach, where clipper levels could change across phases. However,
  they do not pursue this direction (no specific algorithm or results)
  as their objective is to derive gap independent bounds (minimax
  regret).}.
%% that also generalize to a gap independent bounds. 
We now define some key
quantities.

\begin{definition}
Let $f(\cdot)$ be a non-negative convex function such that $f(1) = 0$.
For two joint distributions $p_{X,Y}(x,y)$ and $q_{X,Y}(x,y)$ (and the
associated  conditionals), the conditional $f$-divergence $D_{f}(p_{X
  \vert Y} \Vert q_{X \vert Y})$ is given by: 
\begin{align*}
D_{f}(p_{X \vert Y} \Vert q_{X \vert Y}) = \EE_{q_{X,Y}} \left[f \left( \frac{p_{X \vert Y}(X \vert Y)}{q_{X \vert Y}(X \vert Y)} \right)\right].
\end{align*}
%Further for two discrete conditional distributions $p$ and $q$, $D_{f}(p \Vert q)$ can be expressed as the summation:
%\begin{align*}
%&D_{f}(p(X \vert Y) \Vert q(X \vert Y))\\
%&= \sum_{y} q(y) \left(\sum_{x} f \left( \frac{p( x \vert  y)}{q( x\vert  y)}\right)q(x\vert  y)\right)
%\end{align*}
\end{definition}
 % With some abuse of notation, we will denote $D_f(\mathrm{P}_i(V \vert
 % pa(V)) \Vert \mathrm{P}_j(V \vert pa(V)))$ by $D_f(\mathrm{P}_i \Vert
 % \mathrm{P}_j)$. 
Recall that $\mathrm{P}_i$ is the conditional distribution of node $V$
given the state of its parents $pa(V).$ Thus, $D_f(\mathrm{P}_i \Vert
\mathrm{P}_j)$ is the conditional $f$-divergence between the conditional
distributions $P_i$ and $P_j.$ Now we define some
\textit{log-divergences} that are are crucial in the our analysis. 
% We have the following relation,
% \begin{align} \label{eq:divergence} \mathbb{E}_{i} \left[ \exp
%     \left( \frac{\mathrm{P}_i(V \vert pa(V))}{\mathrm{P}_j(V \vert
%       pa(V))} \right) \right] = \left[ 1+D_{f_1}\left(\mathrm{P}_i
%       \lVert \mathrm{P}_j\right) \right]e. \end{align}
\begin{definition}
	\label{def:mij}
($M_{ij}$ measure)  Consider the function $f_1(x) = x \exp (x-1) -
1$. We define the following log-divergence measure:  $M_{ij} = 1 +
\log (1 + D_{f_1} (\mathrm{P}_i \lVert \mathrm{P}_j)),$ $\forall i,j
\in [K].$ 
%\begin{equation}
%\label{eq:mij}
%
%\end{equation}
\end{definition}
\noindent These log divergences help us in controlling the bias
variance trade-off in importance sampling as shown in
Section~\ref{sec:relation} in the appendix.  We also note that
estimates of the $f$-divergence measure can be had directly from
empirical data (without the knowledge of the full distributions) using
techniques like that of~\cite{perez2008kullback}.

%\textbf{Remark:} We will call this log-divergence measure as the `$M_{ij}$ measure' for lack of a better descriptive name.

%One way to view the information leakage among arms is through the following expression relating the mean under arm $i$ to samples collected under arm $j$,
%\begin{align}
%\label{eq:basic_main}
%\EE_{i}\left[Y\right] = \EE_{j}\left[Y\frac{\mathrm{P}_i(V \vert pa(V))}{\mathrm{P}_j(V \vert pa(V))}\right]
%\end{align}
%which is trivially true as the only change in the joint distribution of all variables is at $\mathrm{P}(V \vert pa(V))$. However, directly using this expression to form an estimator, can lead to arbitrarily bad concentration as the ratio in the expression can be unbounded. Therefore, it is essential to clip the importance weighted samples at a carefully chosen value, such that we balance the trade-off between the variance of the samples and the bias incurred by the clipping. 

\noindent \textbf{Aggregating Interventional Data:} We describe an efficient
estimator of $\EE_{k}[Y]$ ($\forall k \in [K]$) that combines
available samples from different arms. This estimator adaptively
weights samples depending on the relative $M_{ij}$ measures, and also
uses clipping to control variance by introducing bias. The estimator
is given by (\ref{eq:uniestimator_main}).

Suppose we obtain $\tau_{i}$ samples from arm $i \in [K]$. Let the
total number of samples from all arms be denoted by $\tau$. Further,
let us index all the samples by $s \in \{1,2,..,\tau \},$ and
$\mathcal{T}_{k} \subset \{1,2,..,\tau \}$ be the indices of all the
samples collected from arm $k$. Let $X_j(s)$ denotes the sample collected for random variable $X$ under intervention $j$, at time instant $s$. Finally, let $Z_k = \sum_{j \in
  [K]}\tau_j/M_{kj}$. 
We denote the estimate of  $\mu_k$ by $\hat{Y}_{k}^{\epsilon}$
($\epsilon$ indicates the level of confidence desired).
Our estimator is:  
\begin{align*}
\label{eq:uniestimator_main}
&\hat{Y}_{k}^{\epsilon}  = \frac{1}{Z_{k}}\sum_{j = 0}^{K} \sum_{s \in
  \mathcal{T}_j} \frac{1}{M_{kj}}Y_j(s)\frac{\mathrm{P}_k(V_j(s) \vert
  pa(V)_j(s))}{\mathrm{P}_j(V_j(s) \vert pa(V)_j(s))} \times \mathds{1}\left\{ \frac{\mathrm{P}_k(V_j(s) \vert pa(V)_j(s))}{\mathrm{P}_j(V_j(s) \vert pa(V)_j(s))} \leq 2\log(2/\epsilon)M_{kj}\right\}. \numberthis
\end{align*}
In other words, $\hat{Y}_{k}^{\epsilon}$ is the weighted average of
the clipped samples, where the samples from arm $j$ are weighted by
$1/M_{kj}$ and clipped at $ 2\log(2/\epsilon)M_{kj}$. The
\textbf{choice of $\epsilon$} controls the \textit{bias-variance
  tradeoff} which we will adaptively change in our algorithm.

\subsection{Algorithm}
Now, we describe our main algorithmic contribution - Algorithm
\ref{alg:pickbest} and \ref{alg:pickbest2}.  Algorithm
\ref{alg:pickbest}  starts by having all the $K$ arms under
consideration and then proceeds in phases, possibly rejecting one or more arms
at the end of each phase.  

At every phase, Estimator $(\ref{eq:uniestimator_main})$ with a phase
specific choice of  the $\epsilon$ parameter (i.e. controlling bias
variance trade-off), is applied to all arms under consideration. Using
a phase specific threshold on these estimates, some arms are rejected
at the end of each phase. A random arm among the ones surviving at the
end of \textit{all} phases is declared to be the optimal. We now
describe the duration of various phases. 

Recall the parameters $T$ -  Total sample budget available and $B$ -
average cost budget constraint. Let $n(T) = \lceil \log 2 \times \log
10\sqrt{T} \rceil$. Let $\overline{\log}(n) = \sum_{i = 1}^{n}(1/i)$.
We will have an algorithm with $n(T)$ phases numbered by $\ell =
1,2.,,n(T)$. Let $\tau(\ell)$ be the total number of samples in phase
$\ell$. We set $\tau(l) = T/(l\overline{\log}(n(T)))$ for $l \in
\{1,..,n(T) \}$. Note that $\sum_{l}\tau(l)= T$. Let ${\cal R}$ be the
set of arms remaining to compete with the optimal arm at the beginning
of phase $\ell$ which is continuously updated. 

{\noindent \bf Allocation of Budget:} Let $\tau_k(\ell)$ be the samples allocated to arm $k$ in phase
$\ell$. Let $\pmb{\tau(\ell)}$ be the vector consisting of entries $\{
\tau_k\left(\ell \right) \}$. The vector of allocated samples, i.e.
$\pmb{\tau(\ell)}$ is decided by Algorithm~\ref{alg:allocate}. 
%
%% The performance depends crucially on the allocation. 
%
Intuitively, an arm that provides sufficient information about all the
remaining arms needs to be given more budget than other less
informative arms. This allocation depends on the average budget
constraints and the relative log divergences between the arms
(Definition~\ref{def:mij}). Algorithm~\ref{alg:allocate} formalizes
this intuition, and ensures that variance of the worst estimator (of
the form (\ref{eq:uniestimator_main})) for the arms in $\mathcal{R}$
is as \textit{good} as possible (quantified in
Theorem~\ref{thm:uniestimator} and Lemma~\ref{allocation} in the
appendix).

The inverse of the maximal objective of the LP in Algorithm~\ref{alg:allocate} acts as \textit{effective standard deviation} uniformly for all the estimators for the remaining arms in $\mathcal{R}$. It is analogous to the variance terms appearing in Bernstein-type concentration bounds (refer to Lemma~\ref{allocation} in the appendix).

\begin{definition}\label{effectivelogdivergence}
	The \textbf{effective standard deviation} for budget $B$ and arm set ${\cal R} \subseteq [K] $ is defined as ${\sigma^*(B,{\cal R})} = 1/v^*(B,{\cal R})$ from Algorithm~\ref{alg:allocate} with input $B$ and arm set ${\cal R}$.
\end{definition}

Algorithm~\ref{alg:allocate} minimizes the variance terms in the
confidence bounds for the estimates of the arms that are in contention
i.e. $\hat{Y}_{k}^{\epsilon}$ for all $k \in \mathcal{R}$. This
minimization is performed subject to the constraints on the fractional
budget of each arms (recall $B$ from Section~\ref{sec:probdef}). Note
that Algorithm~\ref{alg:allocate} only needs to ensure good confidence
bounds for the arms that are remaining ($k \in \mathcal{R}$). This
gets easier as the number of arms remaining (i.e.
$\lvert \mathcal{R} \rvert$) decreases. Therefore the effective
variances $\sigma^*(B,{\cal R})$ become progressively better with
every phase.

\begin{algorithm}
	\caption{Successive Rejects with Importance Sampling -v1 (SRISv1) - Given total budget $T$ and the cost budget $B$ (along with $c_i$'s) picks the best arm.}
	\begin{algorithmic}[1]
		\STATE $\mathrm{SRIS}(B,\{M_{kj} \},T)$
		\STATE ${\cal R} = [K]$.
		\STATE Form the matrix $\mathbf{A} \in \mathbb{R}^{K \times K}$ such that $A_{kj} = \frac{1}{M_{kj}}$.
		\FOR {$\ell=1$ to $n(T)$}      
		\STATE $\pmb{\tau(\ell)} = \mathrm{ALLOCATE} \left(\mathbf{c}, B, \mathbf{A},{\cal R}, \tau(\ell) \right)$ (Algorithm~\ref{alg:allocate})
		\STATE Use arm $k$, $\tau_k(\ell)$ times and collect samples $(Y,V, pa(V) )$.
		\FOR {$k \in {\cal R}$}
		\STATE \label{samplestamp}Let $\hat{Y}_k$ be the estimator for arm $k$ as in (\ref{eq:uniestimator_main}) calculated with $\{M_{kj} \}$, $\epsilon = 2^{-(\ell-1)}$ and the samples obtained in Line 6.	\ENDFOR
		\STATE Let $\hat{Y}_H  = \argmax _{k \in {\cal R}} \hat{Y}_k$. 
		\STATE ${\cal R}= {\cal R} - \{k \in {\cal R} : \hat{Y}_H > \hat{Y}_k + 5/2^{l} \}$.
		\IF {$\lvert {\cal R}\rvert =1$}
		\STATE \textbf{return:} the arm in ${\cal R}$.
		\ENDIF
		\ENDFOR
		\STATE \textbf{return:} A randomly chosen arm from ${\cal R}$.
	\end{algorithmic}
	\label{alg:pickbest}
\end{algorithm}

\begin{remark} Note that Line 6 uses only the samples acquired in that phase. Clearly, a natural extension is to modify the algorithm to re-use all the samples acquired prior to that step. We give that variation in Algorithm \ref{alg:pickbest2}. We prove all our guarantees for Algorithm~\ref{alg:pickbest}. We conjecture that the second variation has tighter guarantees (dropping a multiplicative log factor) in the sample complexity requirements.
\end{remark}
\begin{algorithm}
	\caption{Successive Rejects with Importance Sampling -v2 (SRISv2) - Given total budget $T$ and the cost budget $B$ (along with $\mathbf{c}$) picks the best arm.}
	\begin{algorithmic}[1]
		\STATE Identical to Algorithm \ref{alg:pickbest} except for Line 6 where all samples acquired in all the phases till that Line is used.
	\end{algorithmic}
	\label{alg:pickbest2}
\end{algorithm}
\label{sec:algo}

\begin{algorithm}
	\caption{Allocate - Allocates a given budget $\tau$ among the arms to reduce variance.}
	\begin{algorithmic}[1]
		\STATE $\mathrm{ALLOCATE}(\mathbf{c},B,\mathbf{A},{\cal R},\tau)$
		\STATE Solve the following LP: 
		\begin{align}
		\label{mainLP1}
		\frac{1}{\sigma^*(B,{\cal R})} &= v^*(B,{\cal R}) =  \max_{\pmb{\nu}} \min_{k \in {\cal R}} [\mathbf{A}\pmb{\nu}]_{k} \\
		&\text{s.t.~} \sum_{i=0}^K c_i \nu_{i} \leq B  \text{ and } \sum_{j = 0}^{K} \nu_j = 1, ~\nu_i \geq 0. \nonumber
		\end{align}
		where $[\mathbf{A}\pmb{\nu}]_{k}$ denotes the $k^{th}$ element in the vector $\mathbf{A}\pmb{\nu}$.
		\STATE Assign $\tau_j = \nu_j^*(B,{\cal R})\tau $
	\end{algorithmic}
	\label{alg:allocate}
\end{algorithm}
\subsection{Theoretical Guarantees}
\label{sec:guarantee}
We state our main results as Theorem~\ref{mainresult} and Theorem~\ref{mainresult2}, which provide guarantees on probability of error and simple regret respectively. Our results can be interpreted as a natural generalization of the results in~\cite{audibert2010best}, when there is information leakage among the arms. This is the first gap dependent characterization.  
%\begin{theorem}\label{mainresult}
%(Proved formally as Theorem \ref{thm:detailed}) Let $\Delta= \min \limits_{k \neq k*} \Delta_k$. Let $\sigma^*$ be the effective variance as in Definition \ref{effectivelogdivergence}. The probability that Algorithm \ref{alg:pickbest} identifies the wrong optimal arm ($e(T,B)$) is at most $\delta$, whenever the following relation involving total number of samples $T$  is satisfied:
%\begin{equation}
% \frac{T}{\overline{\log}(n(T))} \geq  100\left[\log \log(20/\Delta) + \log (2K^2/\delta ) \right]  \bar{H}
%  \end{equation}
% Here, 
% \begin{align}
%  \label{Hbar}
%  \bar{H} = \max_{k \neq k^*} \log_2(10/\Delta_k)^3  \left(\frac{\sigma^*(B,{\cal R}^*(\Delta_k))}{\Delta_k }\right)^2
%  \end{align}
%and ${\cal R}^*(\Delta_k) = \{ s: \log_2 \left(\frac{10}{\Delta_s} \right) \geq \lfloor  \log_2 \left(\frac{10}{\Delta_k }\right) \rfloor  \} $ is the set of arms whose distance from the optimal arm is roughly at most twice that of arm $k$.
%\end{theorem}

\begin{theorem}\label{mainresult}
	(Proved formally as Theorem \ref{thm:detailed}) Let $\Delta= \min \limits_{k \neq k*} \Delta_k$. Let $\sigma^*(.)$ be the effective standard deviation as in Definition \ref{effectivelogdivergence}. The probability of error for  Algorithm \ref{alg:pickbest} satisfies:
	\begin{equation}
	e(T,B) \leq 2K^2 \log_2 (20/\Delta) \exp \left( - \frac{T}{ 2\bar{H} \overline{\log}(n(T))} \right)
	\end{equation}
	when the budget for the total number of samples is $T$ and $\Delta \geq 10/\sqrt{T}$. 
	Here, 
	\begin{align}
	\label{Hbar}
	\bar{H} = \max_{k \neq k^*} \log_2(10/\Delta_k)^3  \left(\frac{\sigma^*(B,{\cal R}^*(\Delta_k))}{\Delta_k }\right)^2
	\end{align}
	and ${\cal R}^*(\Delta_k) = \{ s: \log_2 \left(\frac{10}{\Delta_s} \right) \geq \lfloor  \log_2 \left(\frac{10}{\Delta_k }\right) \rfloor  \} $ is the set of arms whose distance from the optimal arm is roughly at most twice that of arm $k$.
\end{theorem}

\textbf{Comparison with the result in \cite{audibert2010best}:}
Let $\tilde{{\cal R}}(\Delta_k)= \{s: \Delta_s \leq \Delta_k \}$, i.e. the set of arms which are closer to the optimal than arm $k$. Let $\tilde{H} = \max \limits_{k \neq k*} \frac{\lvert \tilde{{\cal R}}(\Delta_k) \rvert}{\Delta_k^2}$. The result in \cite{audibert2010best} can be stated as: \textit{The error in finding the optimal arm is bounded as}: $e(T) \leq O\left(K^2 \exp \left(- \frac{T - K}{\overline{\log}(K)\tilde{H}} \right) \right)$. 

Our work is analogous to the above result (upto poly log factors) except that $\bar{H}$ appears instead of $\tilde{H}$. In Section~\ref{sec:compare1} (in the appendix), we demonstrate through simple examples that $\sigma^*(B,{\cal R}^*(\Delta_k))$ can be significantly smaller than $\sqrt{\tilde{\mathcal{R}}(\Delta_k)}$ (the corresponding term in $e(T)$ above) even when there are no average budget constraints. Moreover, our results can be exponentially better in the presence of average budget constraints (examples in Section~\ref{sec:compare1}). Now we present our bounds on simple regret in Theorem~\ref{mainresult2}.

\begin{theorem}\label{mainresult2}
	(Proved formally as Theorem \ref{thm:detailed})  Let $\sigma^*(.)$ be the effective standard deviation as in Definition \ref{effectivelogdivergence}. The simple regret of Algorithm \ref{alg:pickbest} when the number of samples is $T$ satisfies:
	\ \begin{align}
	&r(T,B) \leq  
	 \frac{10}{\sqrt{T}}\mathds{1} \left\{\exists k\neq k^* \text{ s.t } \Delta_k < 10/\sqrt{T} \right\} +2K^2\sum_{\substack{k \neq k^*:\\ \Delta_k \geq 10/\sqrt{T} }} \Delta_k \log_2\left(\frac{20}{\Delta_k}\right) \exp \left( - \frac{T}{ 2\bar{H}_{k} \overline{\log}(n(T))} \right) 	\label{eq:SRmain}
	\end{align}
	Here, 
%	\begin{align}
%	\label{HK}
	$\bar{H}_{k} = \max_{\{l : \Delta_l \geq \Delta_k\}} \frac{\log_2(10/\Delta_l)^3}{(\Delta_l/10)^2 v^*(B,{\cal R}^*(\Delta_l))^2}$
%	\end{align}
	and ${\cal R}^*(\Delta_k) = \{ s: \log_2 \left(\frac{10}{\Delta_s} \right) \geq \lfloor  \log_2 \left(\frac{10}{\Delta_k }\right) \rfloor  \}. $ 
\end{theorem}

\textbf{Comparison with the result in \cite{lattimore2016causal}:}
In~\cite{lattimore2016causal}, the simple regret scales as
$O(1/\sqrt{T})$ and does not adapt to the gaps. We provide gap
dependent bounds that can be exponentially better than that
of~\cite{lattimore2016causal} (when $\Delta_k$'s are not too small and the first term in (\ref{eq:SRmain}) is zero). Moreover our bounds generalize to gap
independent bounds that match $O(1/\sqrt{T})$. Further details are
provided in Section~\ref{sec:compare2} (in the appendix).

%In~\cite{lattimore2016causal}, the algorithm is based on clipped importance samples, where the clipper is always set at a static level of $O(\sqrt{T})$ (excluding $\log$ factors). The simple regret guarantee in~\cite{lattimore2016causal} scales as $O(\sqrt{(m(\eta)/T)\log T})$, where $m(\eta)$ is a global hardness parameter. The guarantees do not adapt to the problem parameters, specifically the gaps $\{\Delta_k \}_{k \in [K]}$. 
%
%On the contrary, we provide problem dependent bounds. %The terms $\bar{H}_{k}$ can be interpreted as the hardness parameter for rejecting arm $k$. Note that $\bar{H}_k$ depends only on the  arms that are at least as \textit{bad} in terms of their gap from the optimal arm. Moreover the guarantees are adapted to our general budget constraints, which is absent in~\cite{lattimore2016causal}. 
%It can be seen that when $\Delta_k$'s do not scale in $T$, then our simple regret is exponentially small in $T$ (dependent on $\bar{H}_k$'s) and can be much less than $O(1/\sqrt{T})$. The guarantee also generalizes to the problem independent setting when $\Delta_k$'s scale as $O(1/\sqrt{T})$ and can handle general constraints on average budgets. Further details are provided in Section~\ref{sec:compare2}. 

We defer the theoretical analysis to Section~\ref{sec:proofs}. Theorem~\ref{mainresult} and Theorem~\ref{mainresult2} are subparts of our main technical theorem (Theorem~\ref{thm:detailed}), which is proved in Section~\ref{sec:analysis}. 

\section{Empirical Validation}
\label{sec:simulation}
%~\cite{audibert2010best, lattimore2016causal,sachs2005causal}. 
We empirically validate the performance of our algorithms in two real data settings. In Section~\ref{sec:cytometry}, we study the empirical performance of our algorithm on the flow cytometry data-set~\cite{sachs2005causal}.
In Section~\ref{sec:inception}, we apply our algorithms for the purpose of \textit{model interpretability} of the Inception Deep Network~\cite{szegedy2015going} in the context of image classification. Section~\ref{sec:simulationsyn} is dedicated to  synthetic experiments. In Section~\ref{sec:moresims} (in the appendix) we provide more details about our experiments. In the appendix we empirically show that our divergence metric is fundamental and replacing it with other divergences is sub-optimal.
\subsection{Flow Cytometry Data-Set}
\label{sec:cytometry}
The flow cytometry data-set~\cite{sachs2005causal} consists of multi-variate measurements of protein interactions in a \textit{single cell}, under different experimental conditions (\textit{soft interventions}). This data-set has been extensively used for validating causal inference algorithms. Our experiments are aimed at \textit{identifying} the best intervention among many, given some ground truth about the causal graph. For, this purpose we borrow the causal graph from Fig. 5(c) in~\cite{mooij2013cyclic} (shown in Fig.~\ref{fig:cgraph1}) and consider it to be the ground truth.

\begin{figure*}
	\centering
	\subfloat[][{\small Causal Graph for Cytometry Data~\cite{mooij2013cyclic}}]{\includegraphics[width = 0.3\linewidth]{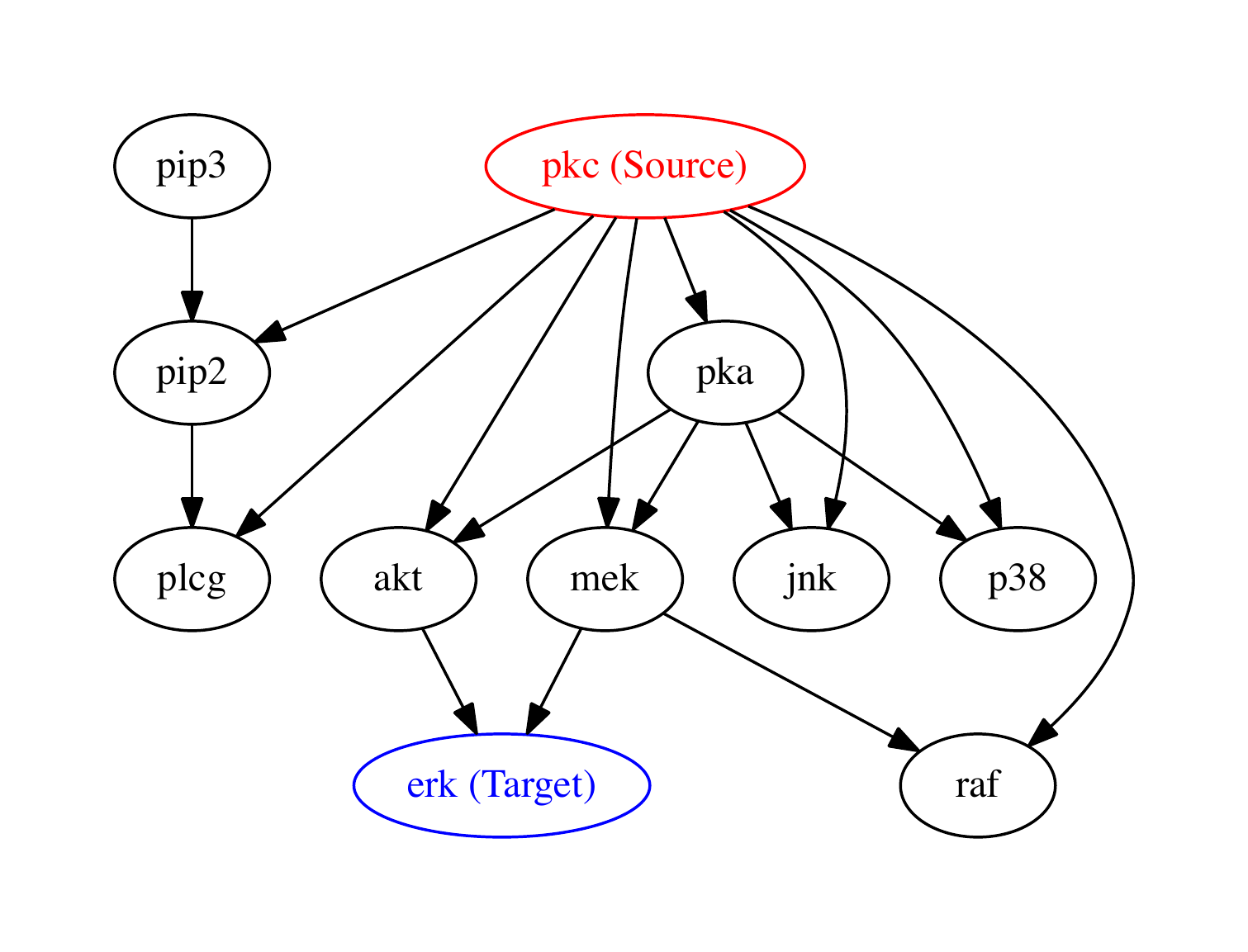}\label{fig:cgraph1}} \hfill
	\subfloat[][{\small Simple Regret when divergences $M_{k0}$'s are high and gap $\Delta$ is small}]{\includegraphics[width = 0.3\linewidth]{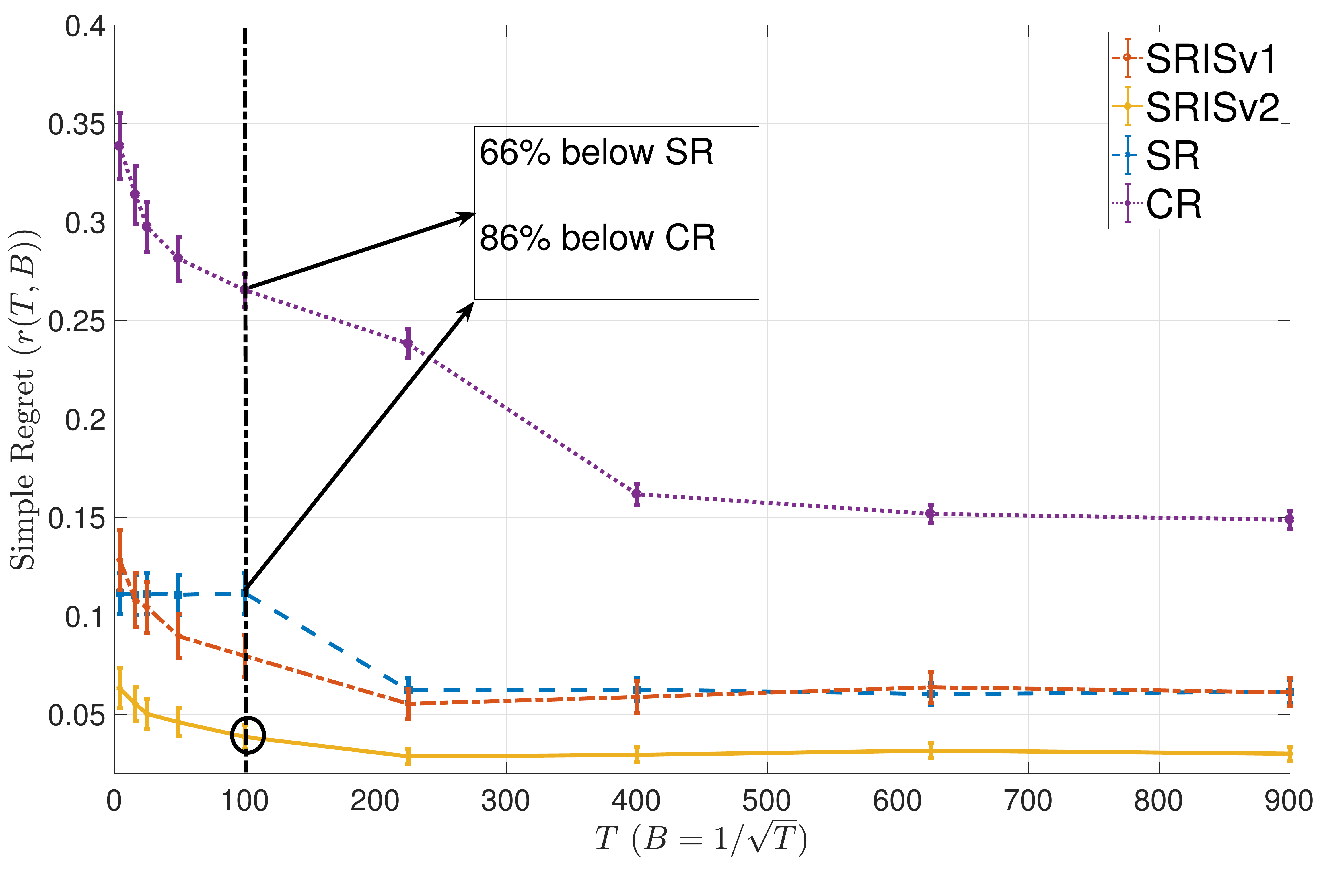}\label{fig:cgraph2}} \hfill
	\subfloat[][{\small Simple Regret when divergences $M_{k0}$'s are low and gap $\Delta$ is small}]{\includegraphics[width = 0.3\linewidth]{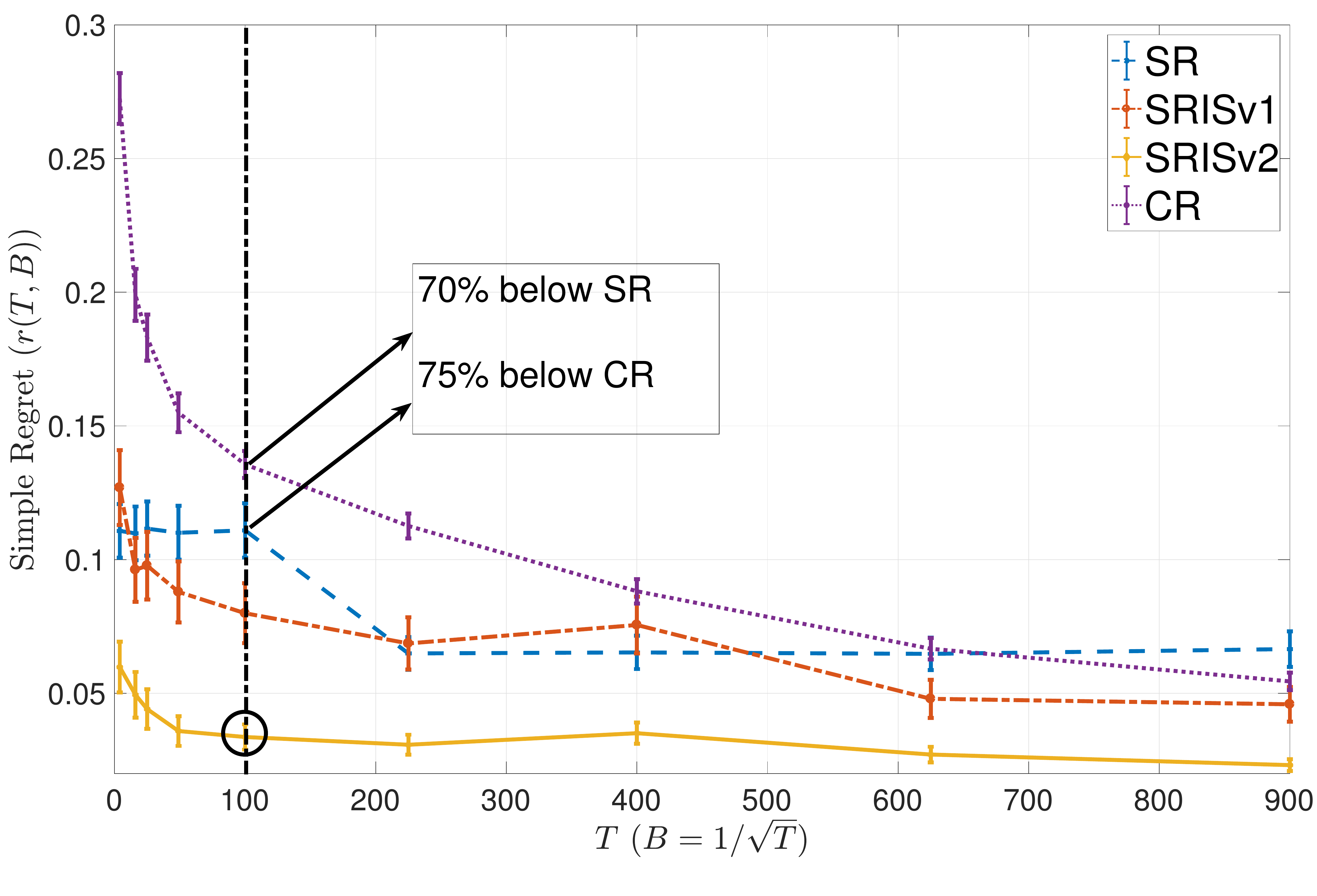}\label{fig:cgraph3}}
	\caption{ \small Performance of various algorithms on the cytometry data under different scenarios. The results are averaged over the course of $500$ independent experiments. The total sample budget $T$ is plotted on the $x$-axis. The budget for all arms other than arm $0$ is constrained to be less than $\sqrt{T}$. Here $K = 15$. The performance improvement is especially evident in the low sample regime. For, example in (b) SRISv2 provides more than 65\% improvement over CR and SR. It is significant in biological data where number of samples is generally small. SR does not use information leakage while the static clipper in CR cannot adapt to high divergences like in (b).}
	\label{fig:cgraph}
\end{figure*} 

Parametric linear models have been popularly used for causal inference on this data-set~\cite{meinshausen2016methods,cho2016reconstructing}. We fit a GLM gamma model~\cite{hardin2007generalized} between the activation of each node and its parents in Fig.~\ref{fig:cgraph1} using the observational data. In Section~\ref{sec:morecyto} (in the appendix) we provide further details showing that the sampled distributions in the fitted model are extremely close to the empirical distributions from the data. The soft interventions signifying the arms are generated by changing the distribution of a source node \textit{pkc} in the GLM. The objective is to identify the intervention that yields the highest output at the target node \textit{erk}\footnote{The activations of the node \textit{erk} have been scaled so that the mean is less than one. Note that the marginal distribution still has an exponential tail, and thus does not strictly adhere to our boundedness assumption on the target variable. However, the experiments suggest that our algorithms still perform extremely well.} We provide empirical results for two sets of interventions at the source node. Both these experiments have been performed with $15$ arms each representing different distributions at \textit{pkc}.

\textbf{Budget Restriction}: The experiments are performed in the budget setting \textbf{S1}, where all arms \textit{except} arm $0$ are deemed to be \textit{difficult}. We plot our results as a function of the total samples $T$, while the fractional budget of the \textit{difficult} arms ($B$) is set to $1/\sqrt{T}$. Therefore, we have $\sum_{k \neq 0} T_k \leq \sqrt{T}$. This essentially belongs to the case when there is a lot of data that can be acquired for a default arm while any new change requires significant cost in acquiring samples.

{\bf  Competing Algorithms: } We test our algorithms on different problem parameters and compare with related prior work~\cite{audibert2010best, lattimore2016causal}. The algorithms compared are (i) {\it SRISv1:}  Algorithm~\ref{alg:pickbest} introduced in Section~\ref{sec:algo}. The divergences, $D_{f_1}(P_i || P_j)$ are estimated from sampled data using techniques from~\cite{perez2008kullback}; (ii){\it SRISv2:}  Algorithm~\ref{alg:pickbest2} as detailed in Section~\ref{sec:algo}; (iii) {\it SR:} Successive Rejects Algorithm from~\cite{audibert2010best} adapted to the budget setting. The division of the total budget $T$ into $K-1$ phases is identical, while the individual arm budgets are decided in each phase according to the budget restrictions; (iv) {\it CR:} Algorithm 2 from~\cite{lattimore2016causal}. The optimization problem for calculating the mixture parameter $\pmb{\eta}$ is not efficiently solvable for general distributions and budget settings. Therefore, the mixture proportions are set by Algorithm~\ref{alg:allocate}.

In these experiments, the budget restrictions imply that arm $0$ can be pulled much more than the other arms. Intuitively the divergences of the arms from arm $0$ as well as the gap $\Delta$ defines the hardness of identification. Fig.~\ref{fig:cgraph2} represents a difficult scenario where the divergences $M_{k0} > 400$ for many arms (large divergences imply low information leakage) and $\Delta=0.01$ (small $\Delta$ increases hardness). In Fig.~\ref{fig:cgraph3} (easier scenario) the divergences $M_{k0} <20$ for most arms while the gap is same as before. We see that SRISv2 outperforms all the other algorithms by a large margin, especially in the low sample regime.

\subsection{Interpretability of Inception Deep Network}
\label{sec:inception}
In this section we use our algorithm for \textit{model interpretation}
of the pre-trained Inception-v3 network~\cite{szegedy2015going} for
classifying images.  \textit{Model Interpretation} essentially
addresses: { \it 'why does a learning model classify in a certain
  way?'}, which is an important question for complicated
models like deep nets~\cite{ribeiro2016should}. 

When an RGB image is fed to Inception, it produces an ordered sequence
of $1000$ labels (e.g 'drums', 'sunglasses') and generally the
top-$10$ labels are an accurate description of the objects in the
image. To address interpretability, we segment the image into a number
of \textit{superpixels}/segments (using segmentation algorithms like
SLIC~\cite{achanta2010slic}) and infer which superpixels encourage the
neural net to output a certain label (henceforth referred to as {\em
  label-I}; e.g 'drum') in top-$k$ (e.g. $k=10$), and to what extent.

\begin{figure}
	\centering
	\subfloat[][{\small Original Image }]{\includegraphics[width = 0.45\linewidth]{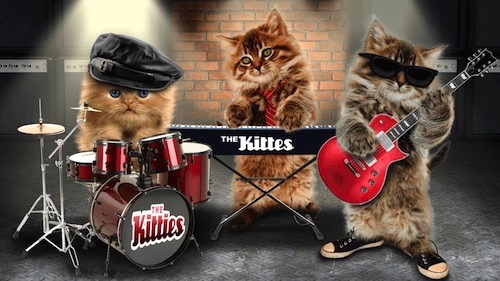}\label{fig:cats1}} \hfill
	\subfloat[][{\small Interpretation- \textit{Drums}}]{\includegraphics[width = 0.45\linewidth]{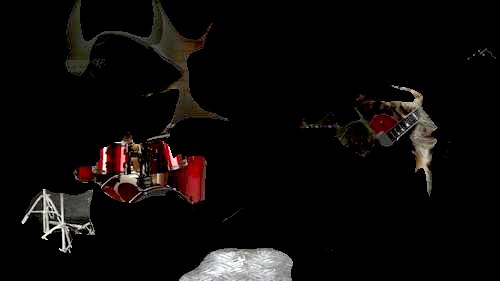}\label{fig:cats2}} \\
	\subfloat[][{\small Interpretation- \textit{TigerCat}}]{\includegraphics[width = 0.45\linewidth]{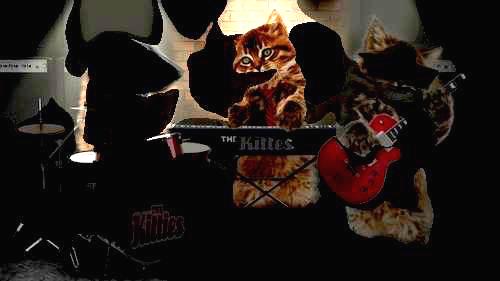}\label{fig:cats3}} \hfill
	\subfloat[][{\small Interpretation- \textit{Sunglasses}}]{\includegraphics[width = 0.45\linewidth]{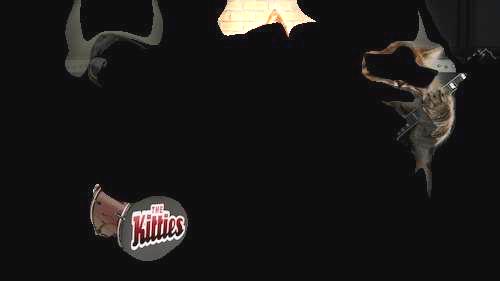}\label{fig:cats4}}
	\caption{\small Interpretation for different top labels for
		the image in (a); Image courtesy \cite{cats-image}. The best
		mixture distribution generates an image where the
		highlighted superpixels are most indicative of a label.
		For example, in (b) we see the
		drums highlighted, while in (d) (a different mixture
		distribution), sunglasses are in the
		focus.} 
	\label{fig:cats}
\end{figure}

Given a \textit{mixture distribution} over the superpixels of an image
(Figure~\ref{fig:cats1}), a few superpixels are randomly sampled
  from the distribution with replacement. Then a new image is
  generated where all other superpixels of the original image are
  blurred out except the ones selected. This image is then fed to
  Inception, and it is observed whether \textit{label-I} appears
  within the top-$k$ labels. A {\em
  mixture distribution is said to be a good interpretation} for
\textit{label-I} if there is a high probability that {\em label-I} appears for an
image generated by this mixture distribution.
To empirically test the goodness of a mixture distribution, we would
generate (using this mixture distribution) a number of random images,
and determine the fraction of images for which {\em label-I} appears;
a large fraction indicates that the mixture
distribution is a good interpretation of {\em label-I}.

% repeatedly appears for a significant
% number of (random) images generated under this mixture distribution.

% if {\em label-I} appears for a significant number of
% images generated by 

% If
% \textit{label I} appears as a top-$k$ label in most images sampled as
% above from a particular mixture distribution, then we say that this
% mixture is a good interpretation for \textit{label I}. 

% For, example if 'drums' appears as one of the top labels for a mixture
% distribution over the superpixels of the image in
% Fig.~\ref{fig:cats1}, then that particular mixture is an
% interpretation for 'drums' in this image.

Motivated by the above discussion, we generate a large number ($3200$)
of mixture distributions, with the goal of finding the one that best
interprets {\em label-I}. To highlight the applicability of our
algorithm, we allow images to be generated for only $200$ of these
mixture distributions; in other words, most of the mixture
distributions cannot be directly tested. Nevertheless, we determine
the best from among the entire collection of mixture
distributions (learning \textit{counterfactuals}).

Specifically in our experiments, we consider the image in
Figure~\ref{fig:cats1}, partition it
into $43$ superpixels, and generate images from mixture distributions by sampling
$5$ superpixels (with replacement). We generate $3000$ arm
distributions which lie in the $43$-dimensional simplex but have
\textit{sparse} support (sparsity of $10$ in our examples). The
support of these distributions are randomly generated by techniques
like markov random walk (encourages contiguity), random choice, etc. as
detailed in Section~\ref{sec:moreincep} in the appendix. However, we
are \textit{only} allowed to sample using a different set of $200$ arms that are dense
distributions chosen uniformly at random from the $43$-dimensional
simplex. The distributions are generated in a manner which is
\textit{completely agnostic} to the image content. The total
sample budget ($T$) is $2500$.

Figure~\ref{fig:cats} shows images in which the segments are weighted
in proportion to the optimal distribution (obtained by SRISv2) for the
interpretation of three different labels. This showcases the true
counterfactual power of the algorithm, as the set of arms that can be
optimized over are \textit{disjoint} from the arms that can be sampled
from. Moreover the sample budget is \textit{less} than the number of
arms. This is an extreme special case of budget setting \textbf{S2}.
We see that our algorithm can generate meaningful interpretations for
all the labels with relatively less number of runs of Inception. Even sampling $10$ times
from each of the arms to be optimized over would require $30,000$ runs
of Inception for a single image and label, while we use only $2500$
runs by leveraging information leakage.

\FloatBarrier
\subsection{Synthetic Experiments}
\label{sec:simulationsyn}
In this section, we empirically validate the performance our algorithm through synthetic experiments. We carefully design our simulation setting which is simple, but at the same time sufficient to capture the various tradeoffs involved in the problem. An important point to note is that our algorithm is not aware of the actual effect of the changes on the target (gaps between expectations) but it only knows the divergence among the candidate soft interventions. Sometimes, a change with large divergence from an existing one may not maximize the effect we are looking for. Conversely, smaller divergence may sometimes lead you closer to the optimal. We demonstrate that our algorithm performs well in all the experiments, as compared to previous works~\cite{audibert2010best, lattimore2016causal}. 

{\bf \noindent Experimental Setup: } We set up our experiments according to the simple causal graph in Figure~\ref{fig:setup}. $V$ is assumed to be a random variable taking values in $\{0,1,2,\cdots, m-1\}$. The various arms $\mathrm{P}_0(V), \mathrm{P}_1(V),...\mathrm{P}_{K-1}(V)$ are discrete distributions with support $[m]$. We will vary $m$ and $K$ over the course of our experiments. 

\begin{figure}
	\centering
	\includegraphics[width=0.5\textwidth]{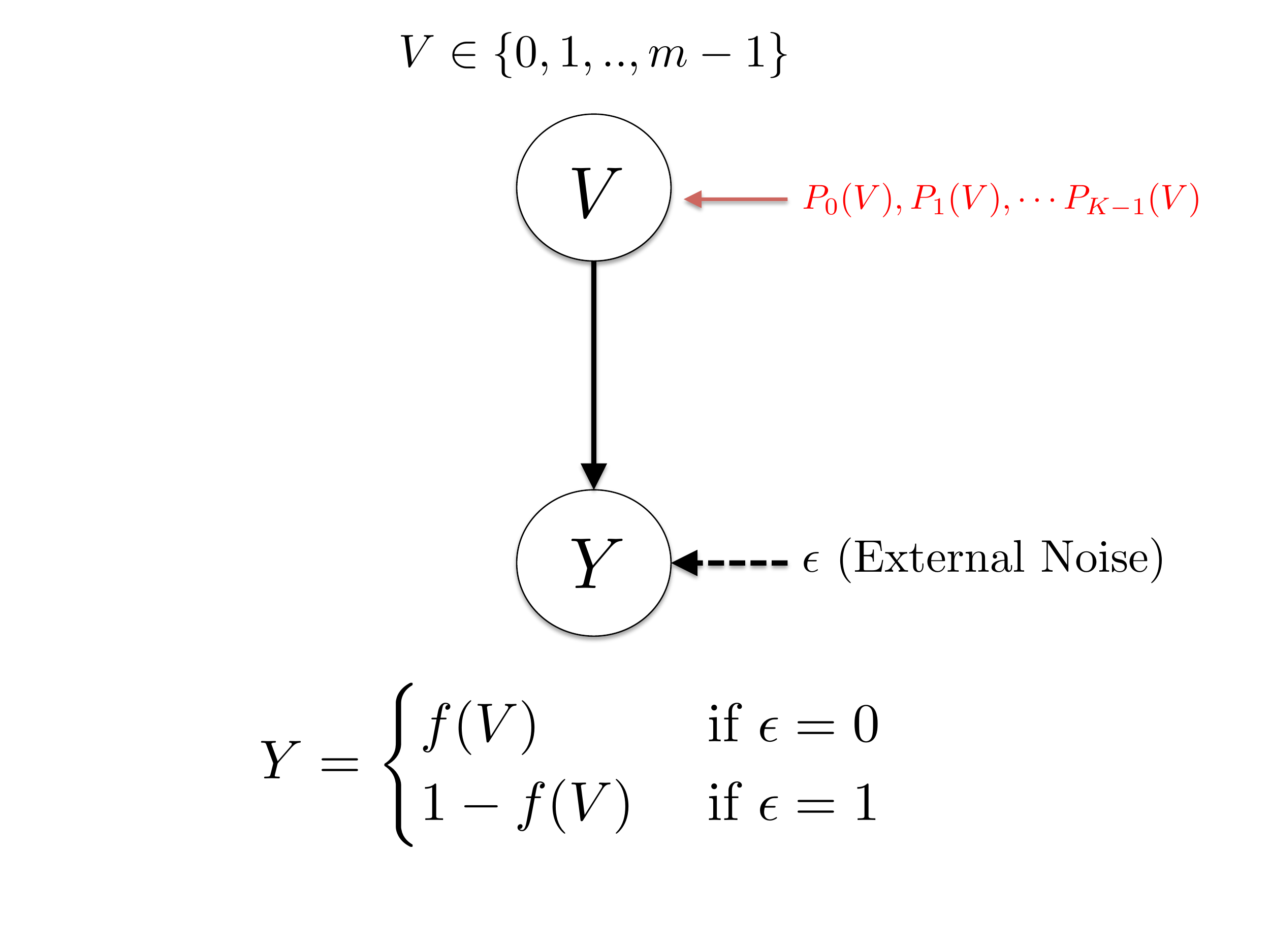}
	\caption{Causal Graph for Experimental Setup}\label{fig:setup}
\end{figure}

$Y$ is assumed to be a function of $V$ and some random noise $\epsilon$ which is external to the system. In our experiments, we set the function as follows:
\[Y =  \begin{cases} 
f(V) & \text{ if } \epsilon = 0\\
1 - f(V) & \text{ if } \epsilon = 1
\end{cases}
\]
where $f : [m] \rightarrow \{0,1 \}$ is an arbitrary function. We set $\PP(\epsilon = 1) = 0.01$ in all our experiments. The discrete candidate distributions are modified to explore various tradeoffs between the gaps and the effective standard deviation parameters.

\noindent \textbf{Budget Restriction}: The experiments are performed in the budget setting \textbf{S1}, where all arms \textit{except} arm $0$ are deemed to be \textit{difficult}. We plot our results as a function of the total samples $T$, while the fractional budget of the \textit{difficult} arms ($B$) is set to $1/\sqrt{T}$. Therefore, we have $\sum_{k \neq 0} T_k \leq \sqrt{T}$. This essentially belongs to the case when there is a lot of data that can be acquired for a default arm while any new change requires significant cost in acquiring samples.

%{\bf \noindent Competing Algorithms: } We test our algorithms on different problem parameters and compare with related prior work~\cite{audibert2010best, lattimore2016causal}. We briefly describe the algorithms compared:
%\begin{enumerate}
%	\item {\bf SRISv1:} This is Algorithm~\ref{alg:pickbest} introduced in Section~\ref{sec:algo}. 
%	\item {\bf SRISv2:} This is Algorithm~\ref{alg:pickbest2} which is a simple modification of SRISv1, as detailed in Section~\ref{sec:algo}. 
%	%, except that the estimators in line 8, use all the samples from the past phases as well, clipped according to the criterion of the current phase. We conjecture that SRISv2 has better theoretical guarantees. We discuss this possible future direction in Section~\ref{sec:conc}. 
%	\item{\bf SR:} This is the best arm identification algorithm from~\cite{audibert2010best} adapted to the budget setting. The division of the total budget $T$ into $K-1$ phases is identical, while the individual arm budgets are decided in each phase according to the budget restrictions.
%	\item{\bf CR:} This is Algorithm 2 from~\cite{lattimore2016causal}. The optimization problem for calculating the mixture parameter $\pmb{\eta}$ has been modified to account for the budget restrictions. This is a natural modification to the algorithm. 
%\end{enumerate}

{\bf \noindent Experiments: } In our experiments, we choose $f$ to be the parity function, when $V \in [m]$, is represented in base $2$. Note that arm $0$ is the arm that can be sampled $O(T)$ times while the rest of the arms can only be sampled $O(\sqrt{T})$ times due to the above budget constraints. So, the divergence of the arm $0$ from other arms  is crucial alongside the gaps. We perform our experiments in different regimes that get progressively easier. In these experiments, we function in various regimes of the divergences between the other arms and arm $0$, and the gaps from the optimal arm in terms of target value. When there is no information leakage, the samples are divided among the $K$ arms. So, the loss in having multiple arms can be expressed as a scaling $\sqrt{K}$ in standard deviation. Recall the $\log$ divergence measure $M_{k0}$ which is a measure of information leakage from arm $0$ to another arm $k$. Therefore, in the following, when we say high divergence from arm $0$, it means that $M_{k0}/\sqrt{K}$ is high for most arms $k \neq 0$. 

{\it High Divergence and Low  $\Delta$: } This is the hardest of all settings. Here, we set $m = 20$ and $K = 30$.  Here, we have $M_{k0}$ to be pretty high for all the arms $k \neq 0$. This means that the arm $0$, which can be pulled $O(T)$ times provides highly noisy estimates for other arms. We have $M_{k0}/\sqrt{K} \sim 30$ for most arms. Moreover, the minimum gap from the best arm $\Delta = 0.04$, which is pretty small. This implies that it is harder to distinguish the best arm. 

The results are demonstrated in Figure~\ref{fig:30arms}. Figure~\ref{fig:30arms1} displays the simple regret. We see that both SRISv1 and SRISv2 outperform the others by a large extent, in this hard setting, even when the number of samples are very low. In Figure~\ref{fig:30arms2} we plot the probability of error in exactly identifying the best arm. We see that none of the algorithms successfully identify the best arm, in the small sample regime, as the gap $\Delta$ is very low. However, our algorithms quickly zero in on arms that are \textit{almost} as good as the optimal, and therefore the simple regret is well-behaved. Our algorithm performs this well even when the divergences are big, because it is able to reject the arms that have high $\Delta_i$ in the early phases, very effectively.  
\begin{figure}
	\centering
	\subfloat[][Simple Regret]{\includegraphics[width = 0.5\linewidth]{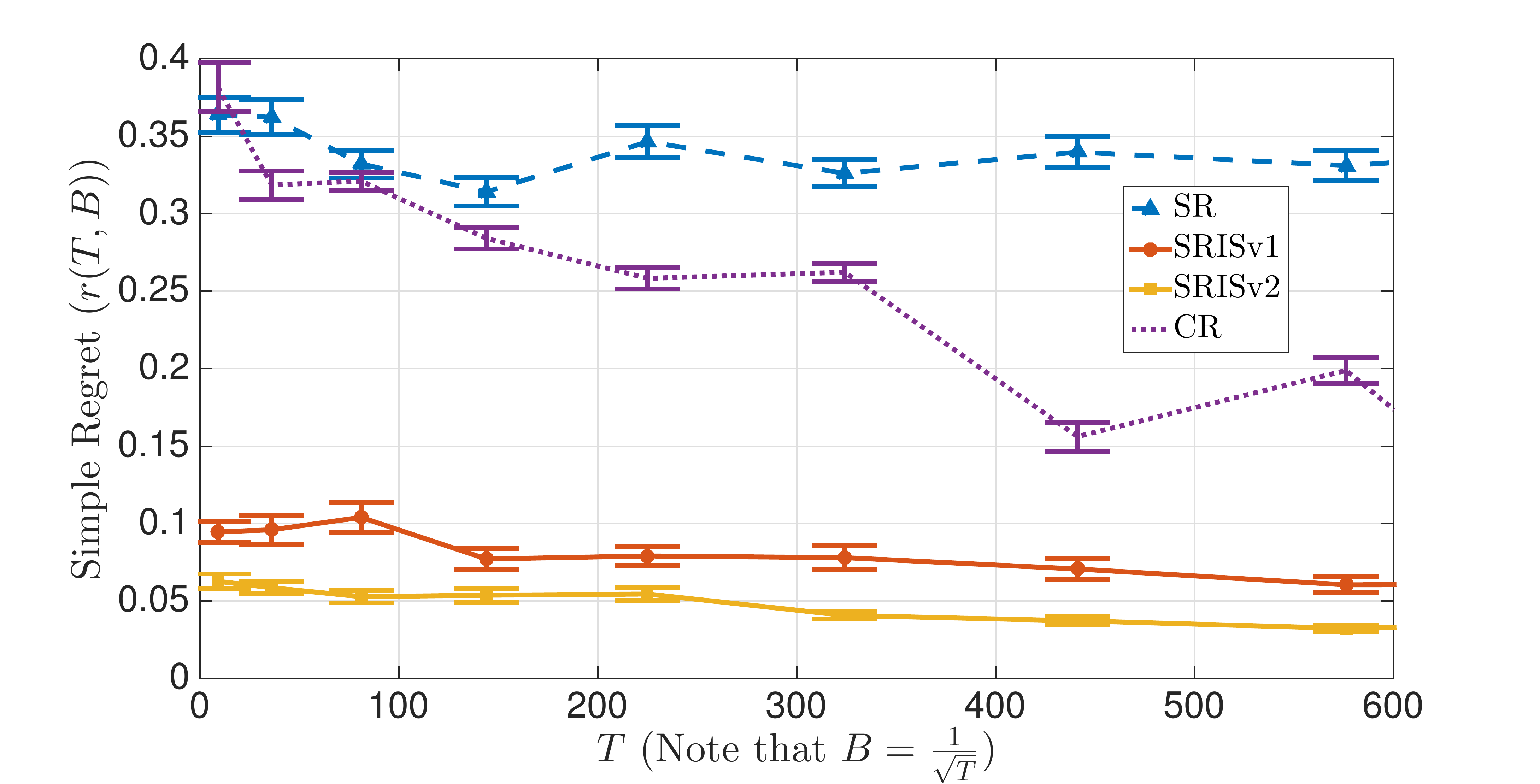}\label{fig:30arms1}}
	\subfloat[][Probability of Error]{\includegraphics[width = 0.5\linewidth]{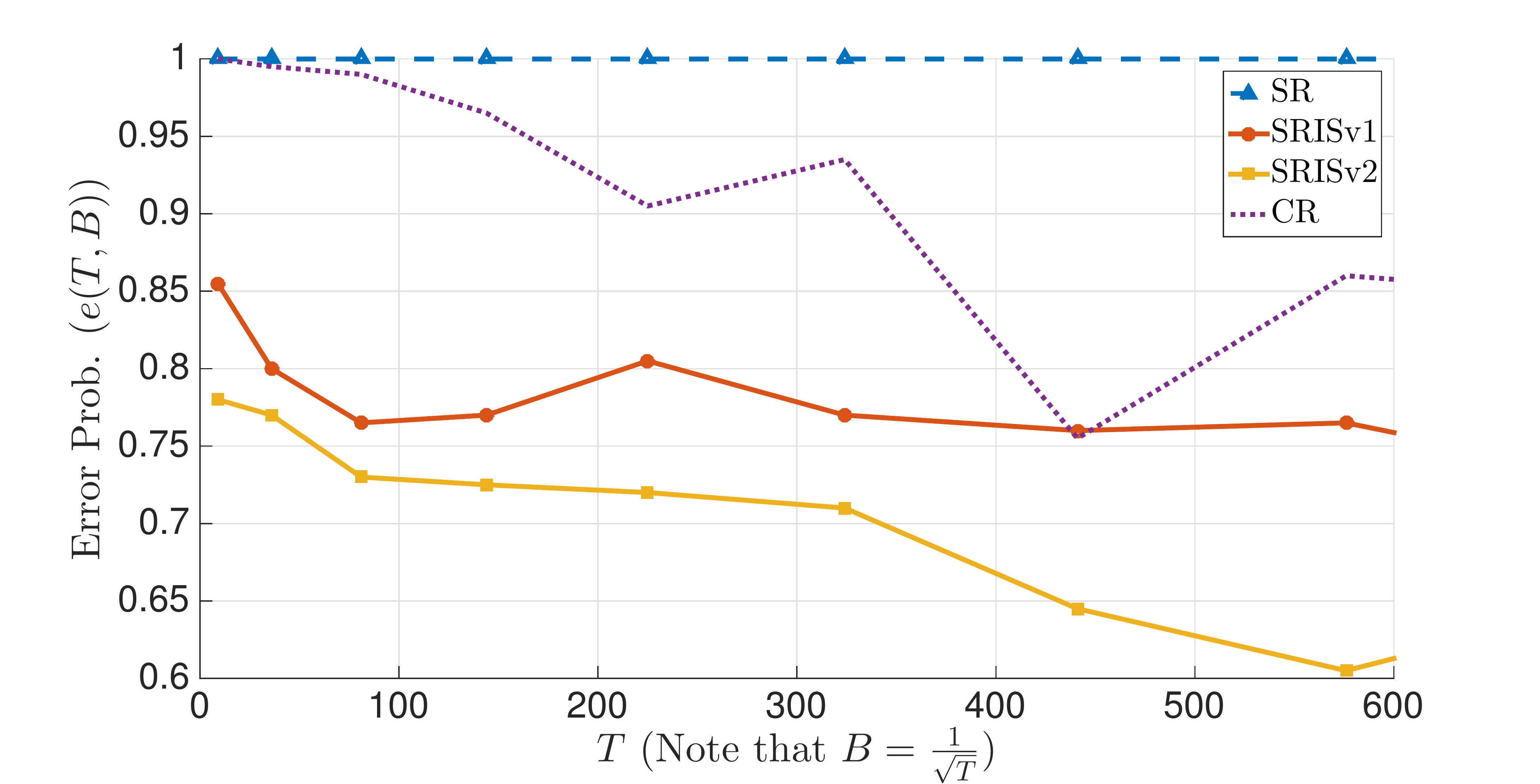}\label{fig:30arms2}}
	\caption{Performance of various algorithms when divergences $M_{k0}$'s are high and mininum gap $\Delta$ is small. The results are averaged over the course of $500$ independent experiments. The total sample budget $T$ is plotted on the $x$-axis. Note that budget for all arms other than arm $0$ is constrained to be less than $\sqrt{T}$. Here $K = 30$.}
	\label{fig:30arms}
\end{figure}

{\it High Divergence and High $\Delta$: } This is easier than the previous setting. Here, we set $m = 10$ and $K = 20$. Here, we have $M_{k0}$ to be very high for all the arms $k \neq 0$. Thus arm $0$ provides very noisy estimates on other arms. We have $M_{k0}/\sqrt{K} \gg 50$ for many arms. However, the minimum gap from the best arm $\Delta = 0.15$, which is not too small. This implies that it might be easier to distinguish the best arm. 

The results are demonstrated in Figure~\ref{fig:20arms}. Figure~\ref{fig:20arms1} displays the simple regret. We see that in the small sample regime SRISv1 and SRISv2 outperform the others by a large extent. In the high sample regime, SRISv2 is still the best, while SR and SRISv1 are close behind. In Figure~\ref{fig:20arms2} we plot the probability of error in exactly identifying the best arm. We see that SRISv2 performance very well in identifying the best arm even though arm $0$ gives highly noise estimates. It is interesting to note that CR does not perform well. This can be attributed to the non-adaptive clipper in CR, that incurs a significant bias because arm $0$ has high-divergences from most of the other arms. 
\begin{figure}
	\centering
	\subfloat[][Simple Regret]{\includegraphics[width = 0.5\linewidth]{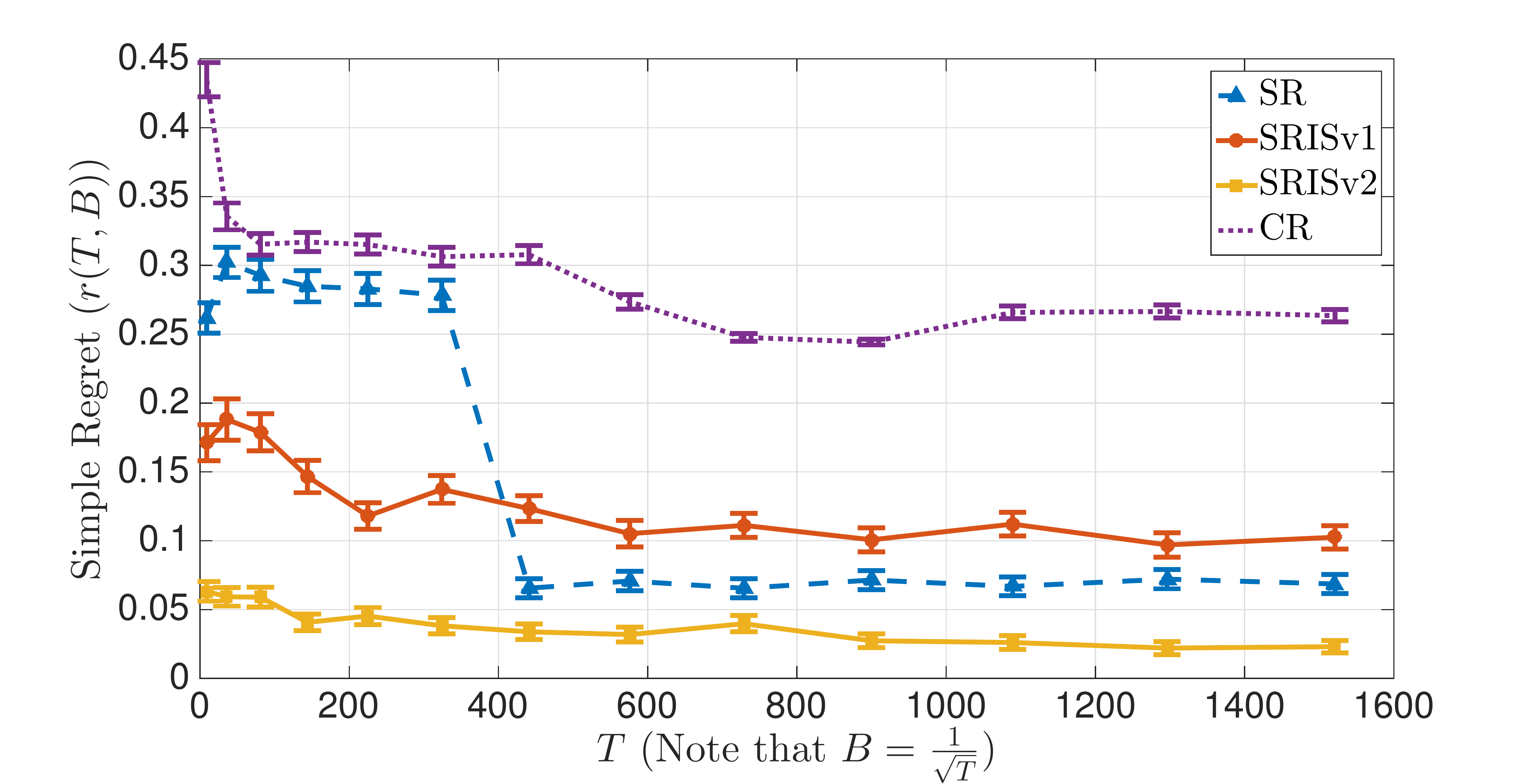}\label{fig:20arms1}}
	\subfloat[][Probability of Error]{\includegraphics[width = 0.5\linewidth]{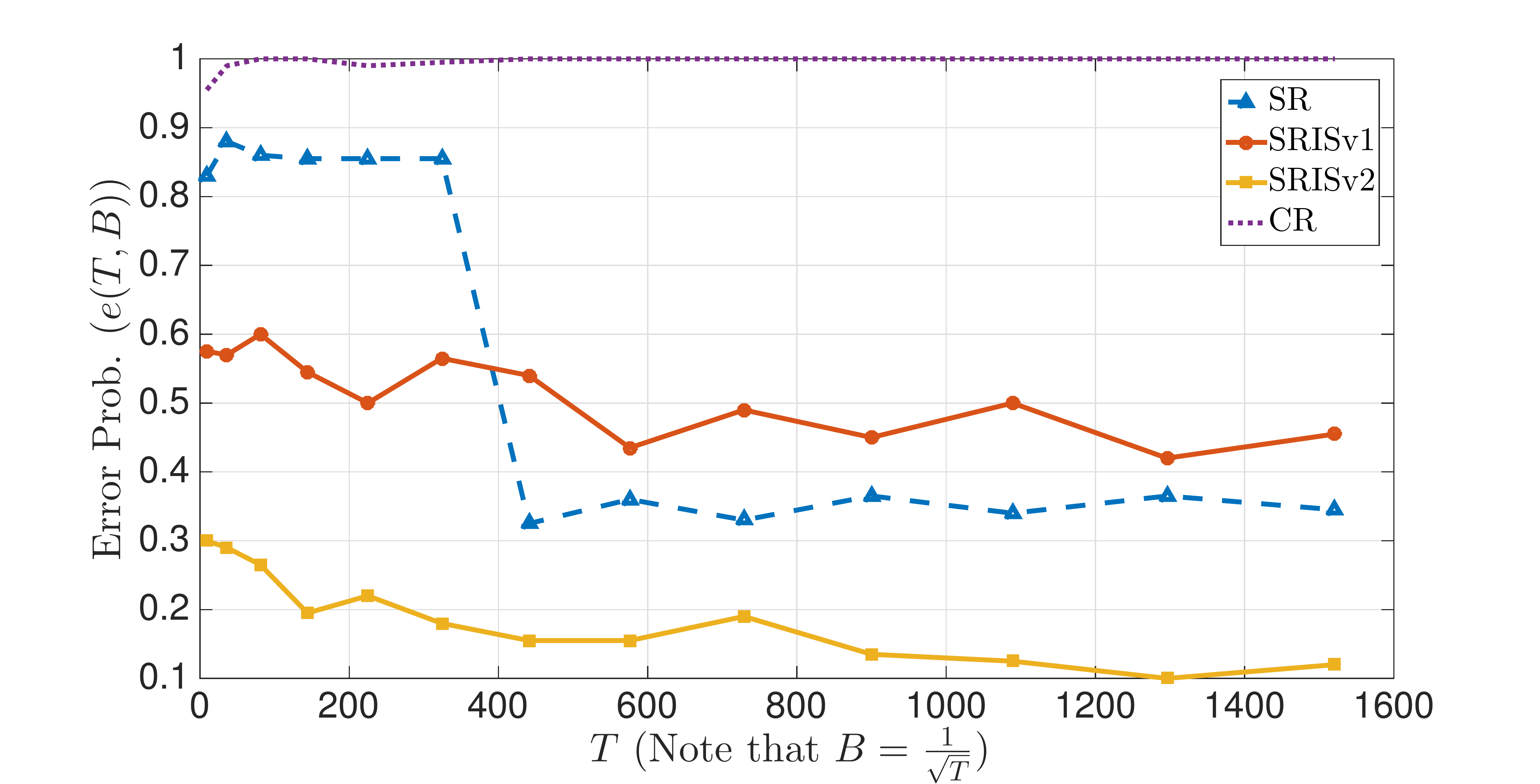}\label{fig:20arms2}}
	\caption{Performance of various algorithms when divergences $M_{k0}$'s are high and min. gap $\Delta$ is not too bad. The results are averaged over the course of $500$ independent experiments. The total sample budget $T$ is plotted on the $x$-axis. Note that budget for all arms other than arm $0$ is constrained to be less than $\sqrt{T}$. Here, $K = 20$. }
	\label{fig:20arms}
\end{figure}

{\it Low Divergence and Low $\Delta$: } This is another moderately hard setting, similar to the previous one. Here, we set $m = 20$ and $K = 30$. Here, we have $M_{k0}$ to be not too high for the arms $k \neq 0$. This means that the arm $0$, which can be pulled $O(T)$ times is moderately good for estimating the other arms. Here, $M_{k0}/\sqrt{K} \leq 10$ for most arms $k$. However, the minimum gap from the best arm $\Delta = 0.04$, which is small. This implies that it might be hard to distinguish the best arm. 

The results are demonstrated in Figure~\ref{fig:30armsN}. Figure~\ref{fig:30armsN1} displays the simple regret. We see that in the small sample regime SRISv1 and SRISv2 outperforms the others by a large extent. In the high sample regime, SRISv2 is still the best, while CR is close behind. In Figure~\ref{fig:30armsN1} we plot the probability of error in exactly identifying the best arm. We see that most of the algorithms have moderately bad probability of error as the gap $\Delta$ is small. However, the algorithms SRISv2 and SRISv1 are quickly able to zero down on arms close to optimal as shown in the simple regret in the small sample regime.  
\begin{figure}
	\centering
	\subfloat[][Simple Regret]{\includegraphics[width = 0.5\linewidth]{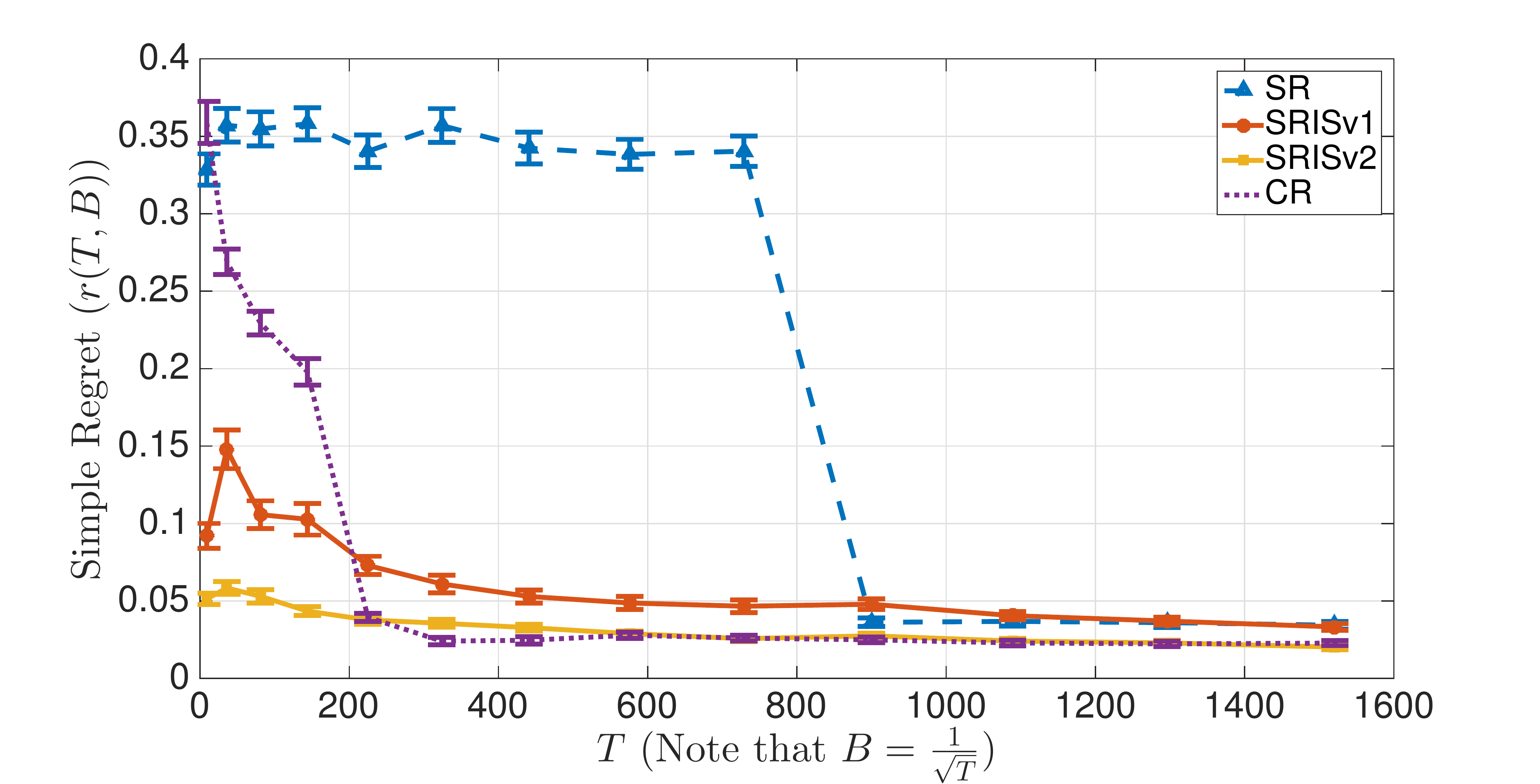}\label{fig:30armsN1}}
	\subfloat[][Probability of Error]{\includegraphics[width = 0.5\linewidth]{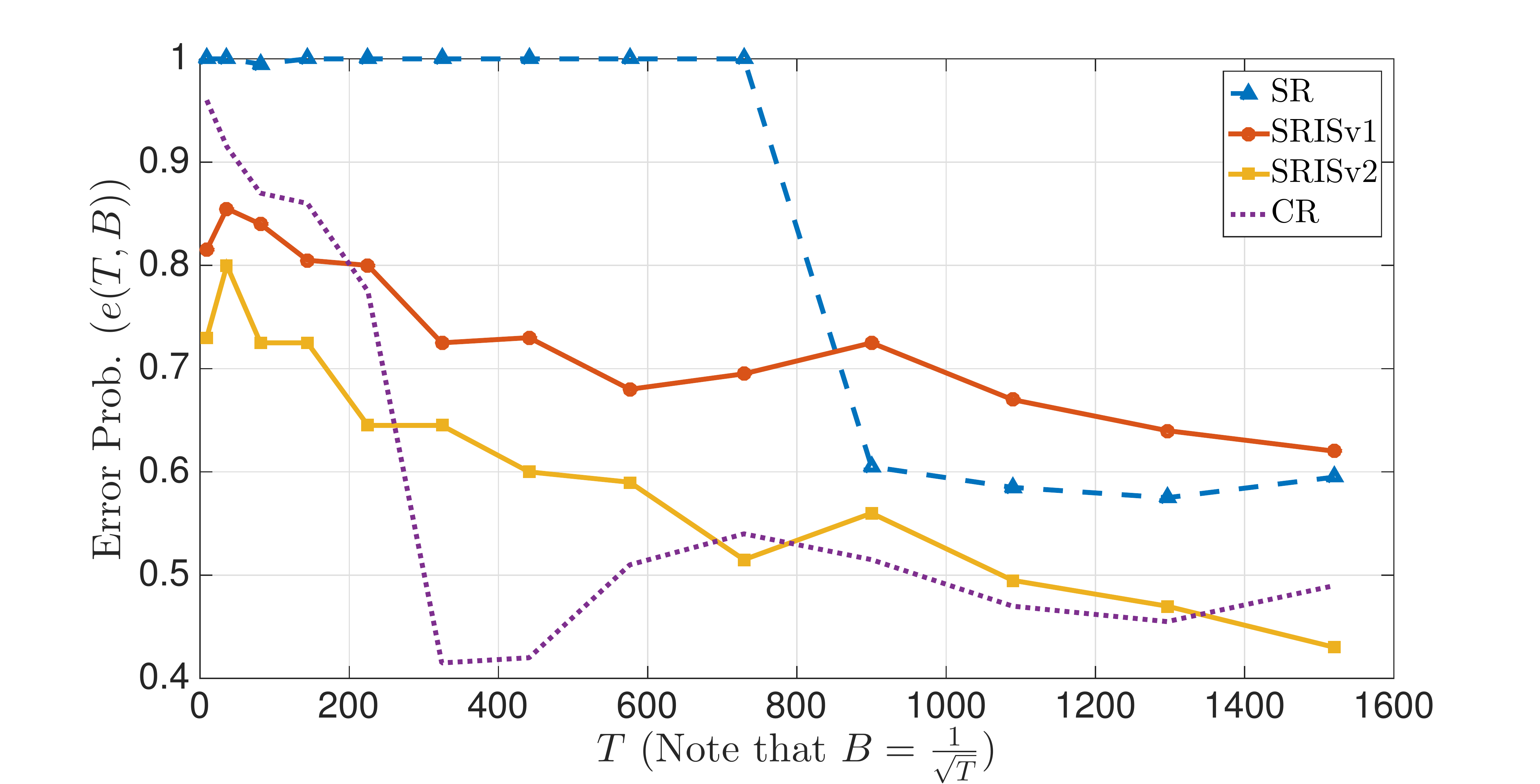}\label{fig:30armsN2}}
	\caption{Performance of various algorithms when divergences $M_{k0}$'s are moderately low and min. gap $\Delta$ is small. The results are averaged over the course of $500$ independent experiments. The total sample budget $T$ is plotted on the $x$-axis. Note that budget for all arms other than arm $0$ is constrained to be less than $\sqrt{T}$. Here, $K = 30$. }
	\label{fig:30armsN}
\end{figure}

{\it Low Divergence and High $\Delta$: } This is the easiest of all settings. Here, we set $m = 10$ and $K = 20$. Here, arms $0$ has $\mathrm{P}_0(V)$ pretty close to the uniform distribution on $[m]$.  Therefore, it is very well-posed for estimating the means of all other arms. In fact we have $M_{k0}/\sqrt{K} < 2$ for many arms. Moreover, the minimum gap from the best arm $\Delta = 0.15$, which is not too small. This implies that it might be very easy to distinguish the best arm. 

The results are demonstrated in Figure~\ref{fig:20armsVE}. Figure~\ref{fig:20armsVE1} displays the simple regret. We see that SRISv2 and CR perform extremely well closely followed by SRISv1. In Figure~\ref{fig:20armsVE2} we plot the probability of error in exactly identifying the best arm. Again SRISv2 and CR have almost zero probability of error and SRISv1 is close behind. This is because $\Delta$ is pretty large. In this example, we observe that all the algorithms that use information leakage are better than SR, because arm $0$ is well-behaved. CR performs almost as well as SRISv2 in this example, as the static clipper is never invoked because almost always the ratios in the importance sampler are well bounded.
\begin{figure}
	\centering
	\subfloat[][Simple Regret]{\includegraphics[width = 0.5\linewidth]{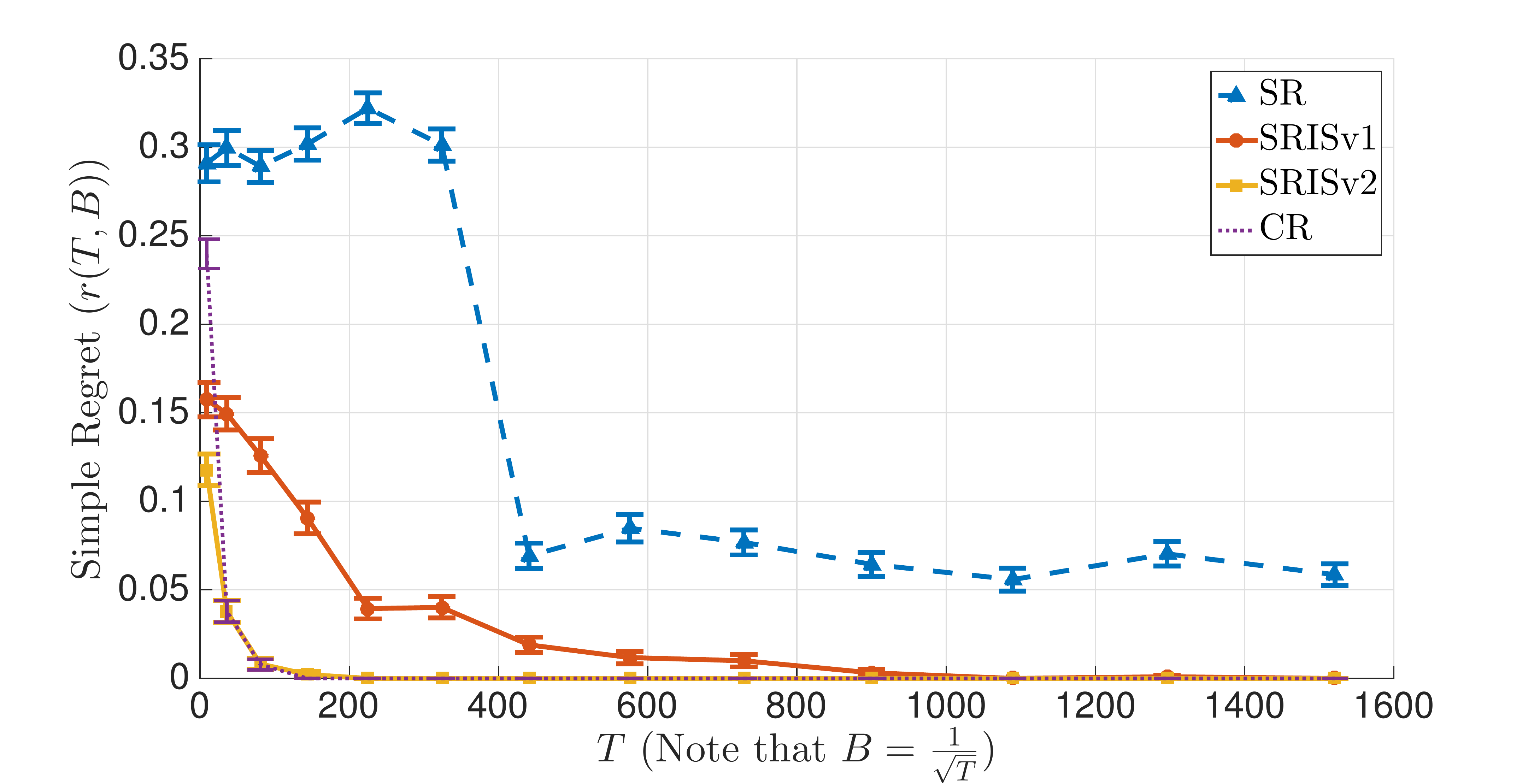}\label{fig:20armsVE1}}
	\subfloat[][Probability of Error]{\includegraphics[width = 0.5\linewidth]{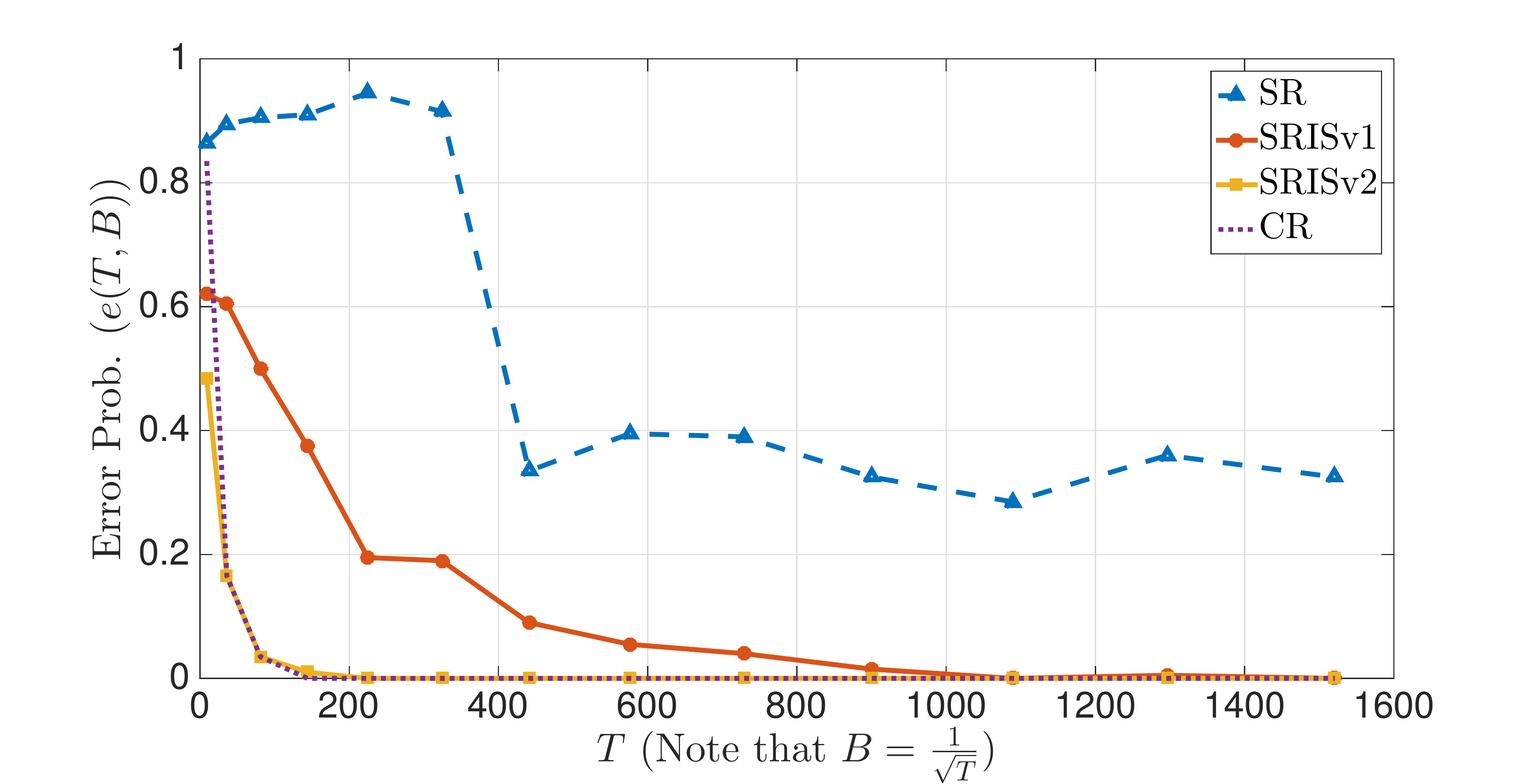}\label{fig:20armsVE2}}
	\caption{Performance of various algorithms when divergences $M_{k0}$'s are low and min. gap $\Delta$ is not too bad. The results are averaged over the course of $500$ independent experiments. The total sample budget $T$ is plotted on the $x$-axis. Note that budget for all arms other than arm $0$ is constrained to be less than $\sqrt{T}$. Here $K = 20$. }
	\label{fig:20armsVE}
\end{figure}

In conclusion, it should be noted that our algorithms perform well in all the different settings, because they are able to adapt to the problem parameters (similar to ~\cite{audibert2010best}) and at the same time leverage the information leakage (similar to ~\cite{lattimore2016causal}).

%\section{Conclusion}
\FloatBarrier
%\section{Conclusion} 
%In this paper we provide the first gap dependent error and simple regret bounds for identifying the best soft intervention at a \textit{source} node i.e the one that maximizes the expected value of a downstream \textit{target} node. These bounds are a generalization of the classical best arm identification bounds~\cite{audibert2010best}, when there is information leakage among the arms. We empirically validate the performance of our algorithm on the flow cytometry data-set. We also use our algorithm for \textit{model interpretation} of Inception-v3 deep net. In the appendix (Section~\ref{sec:detdisc}), we include a detailed discussion on future directions and open problems. %like problem dependent lower bounds, better guarantees for SRISv2 and a general learning framework combining inference and optimization from a causal perspective.
\section{Discussion and Future Work}
\label{sec:detdisc}
In this paper, we analyze the problem of identifying the best arm at a node $V$ in a causal graph (various known conditionals $\mathrm{P}_k(V|pa(V)))$ in terms of its effect on a \textit{target} variable $Y$ further downstream, possibly in a less understood portion of the larger causal network. We characterize the hardness of this problem in terms of the relative divergences of the various conditionals that are being tested and the gaps between the expected value of the target under the various arms. We provide the first problem dependent simple regret and error bounds for this problem, that is a natural generalization of ~\cite{audibert2010best}, but with information leakage between arms. We provide an efficient successive rejects style algorithm that achieves these guarantees, by leveraging the leakage of information, through carefully designed clipped importance samplers. Further, we introduce a new $f$-divergence measure that may be relevant for analyzing importance sampling estimators in the causal context. This may be of independent interest. We believe that our work paves the way for various interesting problems with significant practical implications. In the following, we state a few open questions in this regard:

{\bf Tighter guarantees on SRISv2:} In Section~\ref{sec:simulation}, we have observed that a slightly modified version of our algorithm SRISv2 performs the best among all the competing algorithms including SRISv1. The only difference of SRISv2 from Algorithm~\ref{alg:pickbest}, is that in line 6 the estimators used in a phase also uses samples from past phases, but clipped according to the criterion in the current phase. We believe that this algorithm has tighter error and simple regret guarantees. We conjecture that at least one of the $\log(1/\Delta_k)$, in the definition of $\bar{H}$  in (\ref{Hbar}) can be eliminated, thus leading to better guarantees. 

{\bf Estimating the marginals of the parents:} In Algorithm~\ref{alg:pickbest}, either the marginals of the parents of $V$, that is $\mathrm{P}(pa(V))$ is required in order to calculate  the $f$-divergences in Definition~\ref{def:mij}, or prior data involving the parents is required to estimate the  $f$-divergences directly from data. However, we believe it is possible to model this estimation, directly into the online framework, as data about the marginals of the parents are available through the samples in all the arms, as these marginals remain unchanged. 

{\bf Problem Dependent Lower Bound: } In~\cite{lattimore2016causal}, a problem independent lower bound of $O(1/\sqrt{T})$ has been provided for a special causal graph. However, the problem parameter dependent lower bound like that of~\cite{audibert2010best} still remains an open problem. We believe that the lower bound will depend on the divergences between the distributions and the gaps between the rewards of the arms, similar to the term in (\ref{Hbar}). 

{\bf General Learning Framework:} Our work paves the way for a more general setting for learning counterfactual effects. Importance sampling is a fairly general tool and can be ideally applied at any set of nodes of a causal graph. So, in principle it is possible to study the effect of a change at $V$ on a target $Y$, by using importance sampling between the changed marginal distributions at an intermediate cut ${\cal S}$ that blocks every path from $V$ to $Y$. In fact, this is explored in a non-bandit context in \cite{bottou2013counterfactual}. An important question is: What is the most suitable cut to be used? \cite{lattimore2016causal} uses the cut closest to $Y$, i.e. immediate parents of $Y$. However, the marginals of the cut under different changes need to be estimated this 'far' from the source closer to the target. Therefore, there is a tradeoff that involves a delicate balance between the estimation errors of the changes at an intermediate cut between $V$ and $Y$, and the reduction in importance sampling divergences between cut distributions closer to the target $Y$. We believe understanding this is quite important to fully exploit partial/full knowledge about causal graph structure to answer causal strength questions from data observed.

\bibliography{ImpBandits}
\bibliographystyle{plain}

\clearpage
\appendix

\section{Variations of the Problem Setting}
\label{sec:variation}
In this section we provide more general causal settings where our results can be directly applied.

{\bf Multiple nodes at the graph:} This is illustrated in Fig.~\ref{fig:graphs1}. Soft interventions can be performed at multiple nodes like at $\mathcal{V} = \{V_1, V_2 \}$. These interventions can be modeled as changing the distribution $P(\mathcal{V} \vert pa(\mathcal{V}))$ where $pa(\mathcal{V})$ are the union of parents of $V_1$ and $V_2$. These distributions can be thought of as the arms of the bandits and our techniques can be applied as before to estimate the best intervention.

{\bf Directed cut between sources and targets: } Fig.~\ref{fig:graphs2} represents the most general scenario in which our techniques can be applied. Soft or hard interventions can be performed at multiple \textit{source} nodes, while the goal is to choose the best out of these interventions in terms of maximizing a known function of multiple target nodes. If the effect of these interventions can be estimated on a directed cut separating the targets and the sources then our techniques can be applied as before. This is akin to knowing $P(V_1,V_2)$ under all the interventions in Fig.~\ref{fig:graphs2}, because $V_1$ and $V_2$ is a directed cut separating the sources and the targets. 

{\bf Empirical knowledge of continuous arm distributions: } Our techniques can be applied to continuous distributions $\mathrm{P}(V \vert pa(V))$ as shown in our empirical results in Section~\ref{sec:cytometry}. The extension is straight-forward by using the general definition of $f$-divergences. More importantly our techniques can be applied even if only prior empirical samples from the distributions $\mathrm{P}(V \vert pa(V))$ is available and not the whole distributions. In this case the $f$-divergences can be estimated using nearest neighbor estimators similar to~\cite{perez2008kullback}. Moreover, for the importance sampling only ratios of distributions are needed, which can again be estimated using nearest neighbor based techniques from empirical data.

\begin{figure*}
	\centering
	\subfloat[][]{\includegraphics[width = 0.45\linewidth]{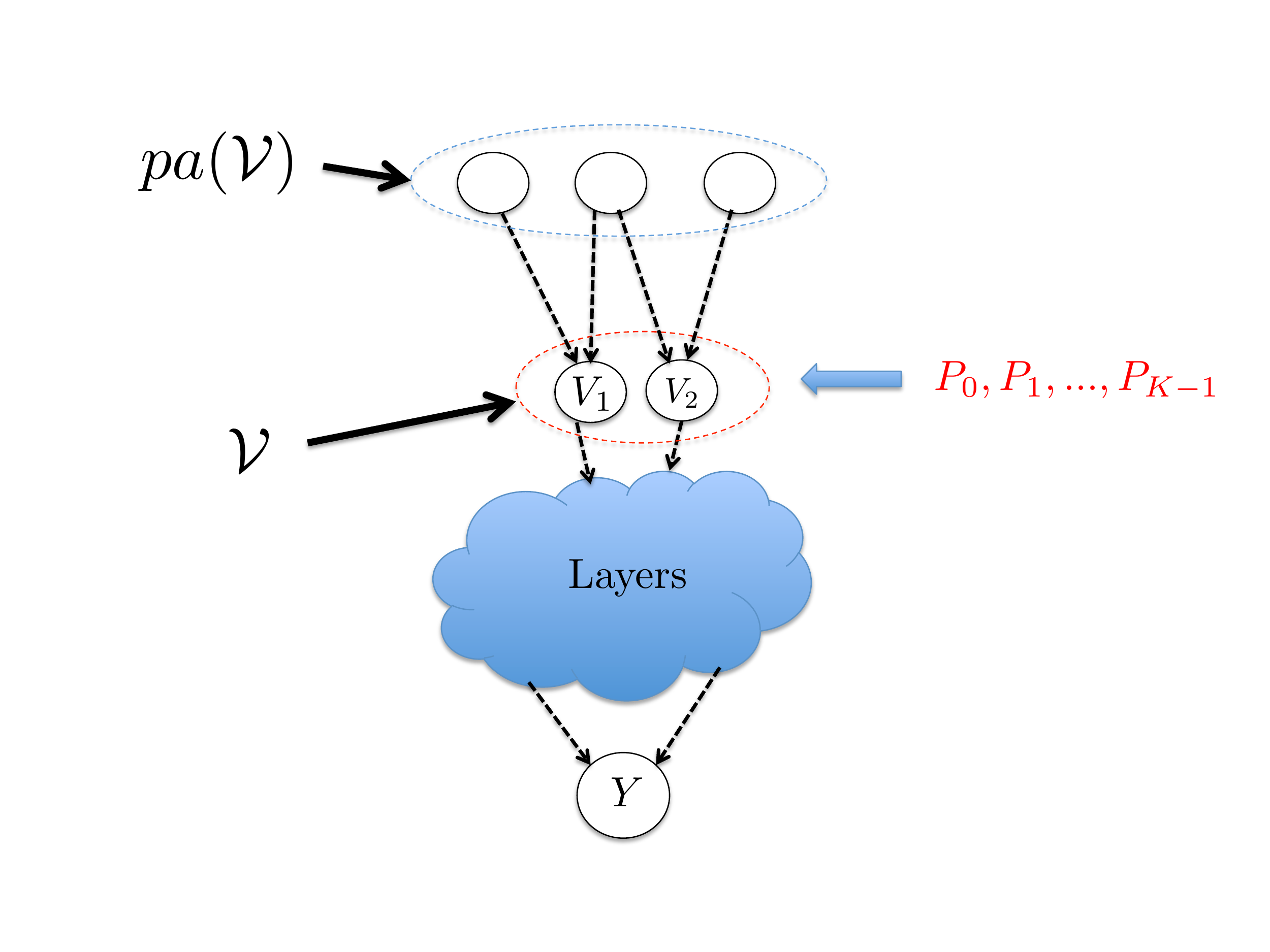}\label{fig:graphs1}} \hfill
	\subfloat[][]{\includegraphics[width = 0.45\linewidth]{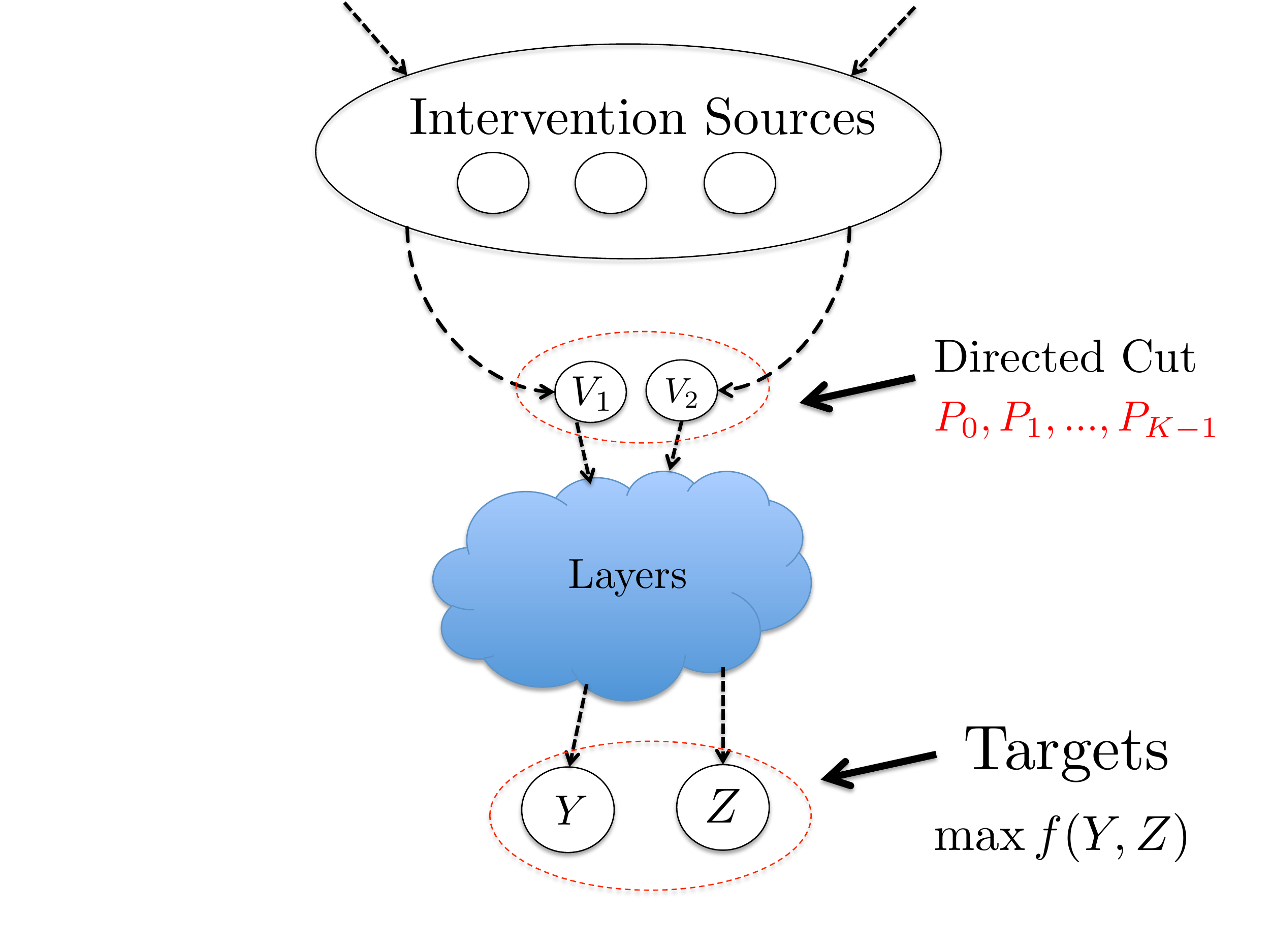}\label{fig:graphs2}} 
	\caption{\small Illustration of more general causal settings where our algorithms can be directly applied. (a) Illustration of a scenario where there are multiple intervention sources $\mathcal{V} = \{V_1, V_2 \}$. Each soft intervention is a change in the distribution $\mathrm{P} (\mathcal{V} \vert pa(\mathcal{V})).$ (b) Illustration of a causal setting, where there are many intervention sources. Soft or hard interventions can be performed at these sources and the effects of these interventions can be observed at the target node. Our techniques can be applied in choosing the best intervention in terms of maximizing a function on the target nodes, provided the effects of these interventions on $P(V_1,V_2)$ are known. Here $V_1$ and $V_2$ form a directed cut separating the sources and the targets }
	\label{fig:graphs}
\end{figure*}

\section{Interpretation of our Theoretical Results}
\label{sec:appen_compare}
In this section we compare our theoretical bounds on the probability of mis-identification with the corresponding bounds in~\cite{audibert2010best}. We also compare our simple regret guarantees with the guarantees in \cite{lattimore2016causal}. In both these cases, we demonstrate significant improvements. These theoretical improvements are exhibited in our empirical results in Section~\ref{sec:cytometry}.
\subsection{Comparison with~\cite{audibert2010best}}
\label{sec:compare1}
Let $\tilde{{\cal R}}(\Delta_k)= \{s: \Delta_s \leq \Delta_k \}$, i.e. the set of arms which are closer to the optimal than arm $k$. Let $\tilde{H} = \max \limits_{k \neq k*} \frac{\lvert \tilde{{\cal R}}(\Delta_k) \rvert}{\Delta_k^2}$. The result for the best arm identification with no information leakage in \cite{audibert2010best} can be stated as: \textit{The error in finding the optimal arm is bounded as}:
\begin{equation}\label{noleakage2}
e(T) \leq O\left(K^2 \exp \left(- \frac{T - K}{\overline{\log}(K)\tilde{H}} \right) \right)
\end{equation}

One intuitive interpretation for $\tilde{H}$ is that it is the maximum among the number of samples (neglecting some log factors) required to conclude that arm $k$ is suboptimal from among the arms which are closer to the optimal than itself. Intuitively, this is because when there is no information leakage, one requires $1/\Delta_k^2$ samples to distinguish between the $k$-th optimal arm and the optimal arm. Further, the $k$th arm is played only $1/k$ fraction of the times since we do not know the identity of the $k$-th optimal arm.

Our main result in Theorem $\ref{mainresult}$ can be seen to be a generalization of the existing result for the case when there is information leakage between the arms (various changes in a causal graph).  

The term $\sigma^*(B,{\cal R}^*(\Delta_k))$ in our setting is the `effective standard deviation' due to information leakage.  There is a similar interpretation of our result (ignoring the log factors): Since there is information leakage, the expression $\frac{(\sigma^*)^2}{(\Delta_k)^2}$ characterizes the number of samples required to weed out arm $k$ out of contention from among competing arms (arms that are at a distance at most twice than that of arm $k$ from the optimal arm). The interpretation of 'effective variance' is justified using importance sampling which is detailed in Section \ref{sec:estimator}. Further, in our framework $\sigma^*$ also incorporates any budget constraint that comes with the problem, i.e. any \textit{apriori} constraint on the relative fraction of times different arms need to be pulled.

For ease of exposition let $(k)$ denote the index of the $k$-th best arm (for $k = 1,..,K$) and $\Delta_{(k)}$ denotes the corresponding gap. In this setting, the terms $\tilde{H}$ (from the result in~\cite{audibert2010best}) and $\bar{H}$ can be written as:
\begin{align*}
\tilde{H} &= \max_{k \neq 1} \frac{k}{\Delta_{(k)}^2} \\
\bar{H} &= \max_{k \neq 1} \frac{\sigma^*(B,{\cal R}^*(\Delta_{(k)}))^2}{\Delta_{(k)}^2}.
\end{align*}
$ \sigma^*(B,{\cal R}^*(\Delta_{(k)}))$ can be smaller than $\sqrt{k}$ due to information leakage as every single arm pull contributes to another arm's estimate. Therefore, these provide better guarantees than ~\cite{audibert2010best}.

To see the improvement over the previous result in \cite{audibert2010best}, we consider a special case when the cost budget $B$ is infinity and there is only the the sample budget $T$. In addition, let us assume that the log divergences are such that: $M_{ij} \leq \eta M_{ii} = \eta << \sqrt{\lvert {\cal R}  \rvert},~ \forall i \neq j$. Let ${\cal R} = {\cal R}^*(\Delta_{(k)})$. If $\eta > \sqrt{\lvert {\cal R}  \rvert}$, the optimal solution for (\ref{mainLP1}) is a bit complicated to interpret.  Consider the feasible allocation $\nu_i = \frac{1}{\lvert {\cal R}  \rvert}, ~ \forall i \in {\cal R}$ in (\ref{mainLP1}). Evaluating the objective function for this feasible allocation, it is possible to show that $\sigma^* \leq \frac{\eta}{1-1/\lvert {\cal R}  \rvert} <<  \sqrt{\lvert {\cal R}  \rvert}$. Hence, unless the variance due to information leakage is too bad, the effective variance is smaller than that of the case with no information leakage.

The improvement over the no information leakage setting, is even more pronounced under budget constraints. Consider the setting \textbf{S1}, and assume that the fractional budget of the \textit{difficult} arms, $B = o(1)$. This implies that the total number of samples available for difficult arms is $o(T)$. The budget constrained case has not been analyzed in~\cite{audibert2010best}, however in the absence of information leakage, one would expect that the arms with the least number of samples would be the most difficult to eliminate, and therefore the error guarantees would scale as $\exp(-O(BT)/\tilde{H}) \sim \exp(-o(T)/\tilde{H})$ (excluding $\log$ factors). On the other hand, our algorithm can leverage the information leakage and the error guarantee would scale as $\exp(-O(T)/\bar{H})$, which can be order-wise better if the \textit{effective standard deviations} are well-behaved. 

\subsection{Comparison with~\cite{lattimore2016causal}}
\label{sec:compare2}
In~\cite{lattimore2016causal}, the algorithm is based on clipped importance samples, where the clipper is always set at a static level of $O(\sqrt{T})$ (excluding $\log$ factors). The simple regret guarantee in~\cite{lattimore2016causal} scales as $O(\sqrt{(m(\eta)/T)\log T})$, where $m(\eta)$ is a global hardness parameter. The guarantees do not adapt to the problem parameters, specifically the gaps $\{\Delta_k \}_{k \in [K]}$. 

On the contrary, we provide problem dependent bounds, which differentiates the arms according to its gap from the optimal arm and its \textit{effective standard deviation} parameter. The terms $\bar{H}_{k}$ can be interpreted as the hardness parameter for rejecting arm $k$. Note that $\bar{H}_k$ depends only on the  arms that are at least as \textit{bad} in terms of their gap from the optimal arm. Moreover the guarantees are adapted to our general budget constraints, which is absent in~\cite{lattimore2016causal}. It can be seen that when $\Delta_k$'s do not scale in $T$, then our simple regret is exponentially small in $T$ (dependent on $\bar{H}_k$'s) and can be much less than $O(1/\sqrt{T})$. The guarantee also generalizes to the problem independent setting when $\Delta_k$'s scale as $O(1/\sqrt{T})$.
\section{Proofs}
In this section we present the theoretical analysis of our algorithm. Before we proceed to the proof of our main theorems, we derive some key lemmas that are useful in analyzing clipped importance sampled estimators. 
\label{sec:proofs}

\subsection{Clipped Importance Sampling Estimator}
\label{sec:divergence}
In Section~\ref{sec:def} we have introduced the concept of importance sampling that helps us in using samples collected under one arm to estimate the means of other arms. As noted in Section~\ref{sec:def}, a naive unbiased importance sampled estimator can potentially have unbounded variances thus leading to poor guarantees. We now introduce clipped importance samplers and provide a novel analysis of these estimators that alleviates the variance related issues.     

{\bf \noindent Clipped Importance Samplers: } The naive estimator of (\ref{eq:naive}) is not suitable for yielding good confidence intervals. It has been observed in the context of importance sampling, that clipping the estimator in (\ref{eq:naive})  at a carefully chosen value, can yield better confidence guarantees even though the resulting estimator will become slightly biased~\cite{bottou2013counterfactual}. Before we introduce the precise estimator, let us define a key quantity that will be useful for the analysis. 
\begin{definition}
\label{def:etas} We define $\eta_{i,j}(\epsilon)$ as follows:
\begin{align}
\eta_{i,j}(\epsilon) = \left\{\argmin_{\eta} : \PP_i\left(\frac{\mathrm{P}_i(V \vert pa(V))}{\mathrm{P}_j(V \vert pa(V))} > \eta \right) \leq \frac{\epsilon}{2} \right\}
\end{align}
for all $i,j \in [K]$, where $\epsilon > 0$. 
\end{definition}
We shall see that the $\eta_{i,j}(\epsilon)$ is related to the conditional $f$-divergence between  $\mathrm{P}_i(V \vert pa(V))$ and $\mathrm{P}_j(V \vert pa(V))$ for the carefully chosen function $f_1(.)$ as introduced in Section~\ref{sec:def}.

Now we are at a position to provide confidence guarantees on the following clipped estimator:
\begin{align*}
&\hat{Y}^{(\eta)}_i(j) = \frac{1}{t}\sum_{s = 1}^{t} Y_j(s)\frac{\mathrm{P}_i(V_j(s) \vert pa(V)_j(s))}{\mathrm{P}_j(V_j(s) \vert pa(V)_j(s))} \times \mathds{1} \left\{ \frac{\mathrm{P}_i(V_j(s) \vert pa(V)_j(s))}{\mathrm{P}_j(V_j(s) \vert pa(V)_j(s))} \leq \eta_{ij}(\epsilon) \right\}. \numberthis \label{eq:goodest}
\end{align*}

\begin{lemma}
\label{lem:clipped1} 
The estimate $\hat{Y}^{(\eta)}_i(j)$ for $\eta = \eta_{i,j}(\epsilon)$ satisfies the following:
\begin{enumerate}
\item \begin{align}
\label{eq:means}
\EE_j\left[\hat{Y}^{(\eta)}_i(j) \right] \leq \mu_{i} \leq \EE_j\left[\hat{Y}^{(\eta)}_i(j) \right] + \frac{\epsilon}{2}
\end{align}
\item \begin{align*}
&\PP \left( \mu_i - \delta -\epsilon/2 \leq  \hat{Y}^{(\eta)}_i(j) \leq \mu_i+\delta \right) \geq 1 - 2\exp \left( -\frac{\delta^{2}t}{2\eta_{i,j}(\epsilon)^2}\right). \numberthis
\end{align*} 
\end{enumerate}
\end{lemma}
\begin{proof}
We have the following chain:
\begin{align*}
&\EE_{j}\left[Y\frac{\mathrm{P}_i(V \vert pa(V)}{\mathrm{P}_j(V \vert pa(V))}\right] \\
&= \EE_{j}\left[Y\frac{\mathrm{P}_i(V \vert pa(V))}{\mathrm{P}_j(V \vert pa(V))}\mathds{1}\left\{ \frac{\mathrm{P}_i(V \vert pa(V))}{\mathrm{P}_j(V \vert pa(V))} \leq \eta_{i,j}(\epsilon)\right\}\right] + \EE_{j}\left[Y\frac{\mathrm{P}_i(V \vert pa(V))}{\mathrm{P}_j(V \vert pa(V))}\mathds{1}\left\{ \frac{\mathrm{P}_i(V \vert pa(V))}{\mathrm{P}_j(V \vert pa(V))} > \eta_{i,j}(\epsilon)\right\}\right] \\
& \overset{(a)} \leq \EE_{j}\left[Y\frac{\mathrm{P}_i(V \vert pa(V))}{\mathrm{P}_j(V \vert pa(V))}\mathds{1}\left\{ \frac{\mathrm{P}_i(V \vert pa(V))}{\mathrm{P}_j(V \vert pa(V))} \leq \eta_{i,j}(\epsilon)\right\}\right] + \PP_i\left(\frac{\mathrm{P}_i(V \vert pa(V))}{\mathrm{P}_j(V \vert pa(V))} > \eta_{i,j}(\epsilon) \right) \nonumber \\
&\leq \EE_{j}\left[Y\frac{\mathrm{P}_i(V \vert pa(V))}{\mathrm{P}_j(V \vert pa(V))}\mathds{1}\left\{ \frac{\mathrm{P}_i(V \vert pa(V))}{\mathrm{P}_j(V \vert pa(V))} \leq \eta_{i,j}(\epsilon)\right\}\right] + \frac{\epsilon}{2}
\end{align*}
Here, $(a)$ is because $Y \in [0,1]$. This yields the first part of the lemma:
\begin{align}
\label{eq:means}
\EE_j\left[\hat{Y}^{(\eta)}_i(j) \right] \leq \mu_{i} \leq \EE_j\left[\hat{Y}^{(\eta)}_i(j) \right] + \frac{\epsilon}{2}
\end{align}
where $\eta = \eta_{i,j}(\epsilon)$. Note that all the terms in the summation of (\ref{eq:goodest}) are bounded by $\eta_{i,j}(\epsilon)$. Therefore, by an application of Azuma-Hoeffding we obtain:
\begin{align}
\label{eq:azuma}
\PP\left( \lvert\hat{Y}^{(\eta)}_i(j)  -  \EE_j\left[\hat{Y}^{(\eta)}_i(j) \right]  \rvert > \delta \right) \leq 2\exp \left( -\frac{\delta^{2}t}{2\eta_{i,j}(\epsilon)^2} \right)
\end{align}
Combining Equation~(\ref{eq:means}) and (\ref{eq:azuma}), we obtain the first part of our lemma. 
\end{proof}

\subsection{Relating $\eta_{ij}(\cdot)$ with $f$-divergence}
\label{sec:relation}

Now we are left with relating $\eta_{i,j}(\epsilon)$ to a particular $f$-divergence ($D_{f_1}$ defined in Section \ref{sec:def}) between $\mathrm{P}_i(V \vert pa(V))$ and $\mathrm{P}_j(V \vert pa(V))$. We have the following relation,
\begin{align}
\label{eq:divergence}
 \mathbb{E}_{i} \left[ \exp \left( \frac{\mathrm{P}_i(V \vert pa(V))}{\mathrm{P}_j(V \vert pa(V))} \right) \right] = \left[ 1+D_{f_1}\left(\mathrm{P}_i \lVert \mathrm{P}_j\right) \right]e.
\end{align}
The following lemma expresses the quantity $\eta_{i,j}(\epsilon)$ as a separable function of $D_{f_1}\left(\mathrm{P}_i \lVert \mathrm{P}_j\right)$ and $\epsilon$, and is one of the key tools used in subsequent analysis. 

\begin{lemma}\label{etalemma}
  It holds that, $\eta_{i,j} (\epsilon) \leq \log \left(\frac{2}{\epsilon}\right) + 1+ \log \left(1+ D_{f_1}\left(\mathrm{P}_i \lVert \mathrm{P}_j\right) \right)$. Furthermore,
 \begin{equation} 
  \eta_{i,j} (\epsilon) \leq 2 \log \left(\frac{2}{\epsilon} \right) \left[1 + \log (1 + D_{f_1} (\mathrm{P}_i \lVert \mathrm{P}_j))  \right] 
  \end{equation}
  when, $\epsilon \leq 1$.
\end{lemma}
\begin{proof}
We have the following chain:
 \begin{align}
  &\mathbb{P}_i \left( \frac{\mathrm{P}_i(V \vert pa(V))}{\mathrm{P}_j(V \vert pa(V))} > \eta \right) \\
  & = \mathbb{P}_i \left( \exp \left(\frac{\mathrm{P}_i(V \vert pa(V))}{\mathrm{P}_j(V \vert pa(V))} \right) > \exp (\eta) \right) \nonumber \\
     \hfill & \overset{(a)} \leq \mathbb{E}_{i} \left[\exp \left(\frac{\mathrm{P}_i(V \vert pa(V))}{\mathrm{P}_j(V \vert pa(V))} \right)\right] \exp (-\eta) \nonumber \\     
  \end{align}
  (a) - We used Markov's inequality. Suppose, we have the right hand side to be at most $\epsilon/2$. Then we have,
  \begin{equation}
  \mathbb{E}_{i} \left[\exp \left(\frac{\mathrm{P}_i(V \vert pa(V))}{\mathrm{P}_j(V \vert pa(V))} \right) \right] \exp (-\eta)  \leq \epsilon/2
\end{equation}
   Now using (\ref{eq:divergence}), we have: 
 \begin{align}
  & \eta \geq \log \left(\frac{2}{\epsilon}\right) + 1 + \log (1 + D_{f_1} (\mathrm{P}_i \lVert \mathrm{P}_j)) \\ \nonumber
   &\implies 
    \mathbb{P}_i \left( \frac{\mathrm{P}_i(V \vert pa(V))}{\mathrm{P}_j(V \vert pa(V))} > \eta \right) \leq \epsilon/2
 \end{align}
 
 From, the definition of $\eta_{i,j} (\epsilon)$, we have:
  \begin{align}
    \eta_{i,j} (\epsilon) &\leq \log \left(\frac{2}{\epsilon}\right) + 1 + \log (1 + D_{f_1} (\mathrm{P}_i \lVert \mathrm{P}_j)) \nonumber \\
    \hfill &\overset{a} \leq 2 \log \left(\frac{2}{\epsilon} \right) \left[1 + \log (1 + D_{f_1} (\mathrm{P}_i \lVert \mathrm{P}_j))  \right], ~\forall \epsilon \leq 1 
  \end{align}
  (a) - This is due to the inequality $p+q \leq 2pq$ when  $q \geq 1$ and $p \geq \log_e(2)$.
\end{proof}  
Now, we introduce the main result of this section as Theorem~\ref{lem:clippedmain}. Recall that $M_{ij} = 1 + \log (1 + D_{f_1} (\mathrm{P}_i \lVert \mathrm{P}_j))$. 
\begin{theorem}
\label{lem:clippedmain} 
The estimate $\hat{Y}^{(\eta)}_i(j)$ for $\eta = 2\log(2/\epsilon)M_{ij}$ satisfies the following confidence guarantees:
\begin{align*}
&\PP \left( \mu_i - \delta -\epsilon/2 \leq  \hat{Y}^{(\eta)}_i(j) \leq \mu_i+\delta \right) \geq 1 - 2\exp \left( -\frac{\delta^{2}t}{8\log(2/\epsilon)^2M_{ij}^2}\right).
\end{align*} 
\end{theorem}
\begin{proof}
The proof is immediate from Lemmas~\ref{etalemma} and \ref{lem:clipped1}. 
\end{proof}

\subsection{Aggregating Heterogenous Clipped Estimators}
\label{sec:estimator}
In Section~\ref{sec:divergence}, we have seen how samples from one of the candidate distribution can be used for estimating the target mean under another arm. Therefore, it is possible to obtain information about the target mean under the $k^{th}$ arm ($\EE_{k}[Y]$) from the samples of all the other arms. It is imperative to design an efficient estimator of $\EE_{k}[Y]$ ($\forall k \in [K]$) that seamlessly uses the samples from all arms, possibly with variable weights depending on the relative divergences between the distributions. In this section we will come up with one such estimator, based on the insight gained in Section~\ref{sec:divergence}. 

Recall the quantities $M_{kj} = 1 + \log (1 + D_{f_1} (\mathrm{P}_k \lVert \mathrm{P}_j))$ ($\forall k,j \in [K]$). These quantities will be the key tools in designing the estimators in this section. Suppose we obtain $\tau_{i}$ samples from arm $i \in [K]$. Let the total number of samples from all arms put together be $\tau$. 

%Therefore, we have the following relation:
%\begin{align*}
%\sum_{i \in \mathcal{B}} \tau_{i} &\leq \beta \\
%\sum_{i \in [K]} \tau_{i} &\leq \tau
%\end{align*} 

Let us index all the samples by $s \in \{1,2,..,\tau \}$. Let $\mathcal{T}_{k} \subset \{1,2,..,\tau \}$ be the indices of all the samples collected from arm $k$. Further, let $Z_k = \sum_{j \in [K]}\tau_j/M_{kj}$. Now, we are at the position to introduce the estimator for $\mu_k$, which we will denoted by $\hat{Y}_{k}^{\epsilon}$ ($\epsilon$ is an indicator of the level of confidence desired):
\begin{align*}
\label{eq:uniestimator}
&\hat{Y}_{k}^{\epsilon}  = \frac{1}{Z_{k}}\sum_{j = 0}^{K} \sum_{s \in \mathcal{T}_j} \frac{1}{M_{kj}}Y_j(s)\frac{\mathrm{P}_k(V_j(s) \vert pa(V)_j(s))}{\mathrm{P}_j(V_j(s) \vert pa(V)_j(s))} \times
\mathds{1}\left\{ \frac{\mathrm{P}_k(V_j(s) \vert pa(V)_j(s))}{\mathrm{P}_j(V_j(s) \vert pa(V)_j(s))} \leq 2\log(2/\epsilon)M_{kj}\right\}. \numberthis
\end{align*}
In other words, $\hat{Y}_{k}^{\epsilon}$ is the weighted average of the clipped samples, where the samples from arm $j$ are weighted by $1/M_{kj}$ and clipped at $ 2\log(2/\epsilon)M_{kj}$.

\begin{lemma}
\begin{align}
\label{eq:means_all}
\hat{\mu}_{k} := \EE\left[\hat{Y}^{\epsilon}_k \right] \leq \mu_{k} \leq \EE\left[\hat{Y}^{\epsilon}_k \right] + \frac{\epsilon}{2}
\end{align} 
\end{lemma}
\begin{proof}
We note that $\hat{Y}_{k}^{\epsilon}$ can be written as:
\begin{align}\label{aggest}
\hat{Y}_{k}^{\epsilon} =  \frac{1}{Z_{k}}\sum_{j = 0}^{K} \frac{\tau_j}{M_{kj}} \tilde{Y}_{kj}^{\epsilon}
\end{align}
Here, $\tilde{Y}_{kj}^{\epsilon} =   \frac{1}{\tau_j}\sum_{s \in \mathcal{T}_j}Y_j(s)\frac{\mathrm{P}_k(V_j(s) \vert pa(V)_j(s))}{\mathrm{P}_j(V_j(s) \vert pa(V)_j(s))}$ $\times \mathds{1}\left\{ \frac{\mathrm{P}_k(V_j(s) \vert pa(V)_j(s))}{\mathrm{P}_j(V_j(s) \vert pa(V)_j(s))} \leq 2\log(2/\epsilon)M_{kj}\right \}$.
Using Lemma~\ref{etalemma} it is easy to observe that $\mathbb{E} [\tilde{Y}_k^{\epsilon}] \leq \mu_k \leq  \mathbb{E} [\tilde{Y}_k^{\epsilon}] +  \frac{\epsilon}{2}$ as $\eta_{kj}(\epsilon) \leq 2\log(2/\epsilon)M_{kj}$. Now, (\ref{aggest}) together with this implies the lemma as $Z_k = \sum_{j \in [K]}\tau_j/M_{kj}$.

\end{proof}

\begin{theorem}
\label{thm:uniestimator}
The estimator $\hat{Y}_{k}^{\epsilon}$ of (\ref{eq:uniestimator}) satisfies the following concentration guarantee:
\begin{align*}
 &\PP \left(\mu_k -\delta - \epsilon/2  \leq \hat{Y}^{\epsilon}_k \leq \mu_k + \delta \right) \geq 1 - 2\exp \left( -\frac{\delta^{2}\tau}{8 (\log (2/\epsilon))^2 } \left(\frac{Z_k}{\tau} \right)^2\right) 
\end{align*}
\end{theorem}

\begin{proof}
For the sake of analysis, let us consider the rescaled version $\bar{Y}_{k}^{\epsilon} = (Z_k/\tau)\hat{Y}_{k}^{\epsilon}$ which can be written as:
\begin{align*}
\label{eq:scaledestimator}
&\bar{Y}_{k}^{\epsilon} = \frac{1}{\tau}\sum_{j = 0}^{K} \sum_{s \in \mathcal{T}_j} \frac{1}{M_{kj}}Y_j(s)\frac{\mathrm{P}_k(V_j(s) \vert pa(V)_j(s))}{\mathrm{P}_j(V_j(s) \vert pa(V)_j(s))}\times\mathds{1}\left\{ \frac{\mathrm{P}_k(V_j(s) \vert pa(V)_j(s))}{\mathrm{P}_j(V_j(s) \vert pa(V)_j(s))} \leq 2\log(2/\epsilon)M_{kj}\right\}. \numberthis
\end{align*}
Since $Y_j(s) \leq 1$, we have every random variable in the sum in (\ref{eq:scaledestimator}) bounded by $2 \log (2/\epsilon)$

Let, $\bar{\mu}_k = \mathbb{E} [\bar{Y}_{k}^{\epsilon}]$.
Therefore by Chernoff's bound, we have the following chain:
\begin{align}
\label{eq:fullchernoff}
& \PP \left( \lvert \bar{Y}_k - \bar{\mu}_k \rvert \leq   \delta \right) \leq 2\exp \left( -\frac{\delta^2\tau}{8 (\log (2/\epsilon))^2}\right)  \nonumber \\
&\implies \PP \left( \lvert \bar{Y}_k \frac{\tau}{Z_k} -\bar{\mu}_k \frac{\tau}{Z_k} \rvert \leq \delta \frac{\tau}{Z_k} \right) \leq 2\exp \left( -\frac{\delta^2\tau}{8 (\log (2/\epsilon))^2}\right) \nonumber \\
&\implies \PP \left( \lvert \hat{Y}_k - \hat{\mu}_k \rvert \leq \delta \frac{\tau}{Z_k} \right) \leq 2\exp \left( -\frac{\delta^{2}\tau}{8 (\log (2/\epsilon))^2}\right) \nonumber \\
&\implies \PP \left( \lvert \hat{Y}_k - \hat{\mu}_k  \rvert \leq \delta \right) \leq 2\exp \left( -\frac{\delta^{2}\tau}{8 (\log (2/\epsilon))^2 } \left(\frac{Z_k}{\tau} \right)^2\right) 
\end{align}
Now we can combine Equations (\ref{eq:fullchernoff}) and (\ref{eq:means_all}) we get:
\begin{align*}
 &\PP \left(\mu_k -\delta - \epsilon/2  \leq \hat{Y}^{\epsilon}_k \leq \mu_k + \delta \right) \geq 1 - 2\exp \left( -\frac{\delta^{2}\tau}{8 (\log (2/\epsilon))^2 } \left(\frac{Z_k}{\tau} \right)^2\right) 
\end{align*}
\end{proof}
In Theorem~\ref{thm:uniestimator}, we observe that the first part of the exponent scales as $O(\epsilon^2\tau/(\log(2/\epsilon))^2)$ if we set $\delta = O(\epsilon)$, which is very close to the usual Chernoff's bound with $\tau$ i.i.d samples. The performance of this estimator therefore depends on the factor $(Z_k/\tau)$ which depends on the \textit{fixed} quantities $M_{kj}$ ($\forall j$) and the allocation of the samples $\tau_{j}$. In the next section, we will come up with a strategy to allocate the samples so that the estimators $\hat{Y}_{k}^{\epsilon}$ have good guarantees for all the arms $k$.

\subsection{Allocation of Samples}
\label{sec:LP}
In Section~\ref{sec:estimator}, Theorem \ref{thm:uniestimator} tells us that the confidence guarantees on the estimator depends on how the samples are allocated between the arms. To be more precise, the term $(Z_k/\tau)$ in Equation~(\ref{eq:fullchernoff}), affects the performance of the estimator for $\mu_k$ ($\hat{Y}_{k}^{\epsilon}$). We would like to maximize  $(Z_k/\tau)$ for all arms $k \in [K]$.

Let the total budget be $\tau$. Let ${\cal R}$ be the set of arms that remain in contention for the best optimal arm. Consider the matrix $\mathbf{A} \in \mathbb{R}^{K \times K}$ such that $\mathbf{A}_{kj} = 1/M_{kj}$ for all $k,j \in [K]$.  Then, we decide the fraction of times arm $k$ gets pulled, i.e. $\nu_k$ to maximize $Z_k$ using the Algorithm~\ref{alg:allocate}.

%\begin{algorithm}
%  \caption{Allocate - Allocates a given budget $\tau$ among the arms to reduce variance.}
%  \begin{algorithmic}[1]
%    \STATE $\mathrm{Allocate}(\mathbf{c},B,\mathbf{A},{\cal R},\tau)$
%    \STATE Solve the following LP: 
%    \begin{align}
%\label{mainLP1}
%\frac{1}{\sigma^*(B,{\cal R})} &= v^*(B,{\cal R}) \max_{\pmb{\nu}} \min_{k \in {\cal R}} [\mathbf{A}\pmb{\nu}]_{k} \\
% &\text{s.t.~} \sum_{i=0}^K c_i \nu_{i} \leq B  \text{ and } \sum_{j = 0}^{K} \nu_j = 1, ~\nu_i \geq 0. \nonumber
%\end{align}
%    \STATE Assign $\tau_j = \nu_j^*(B,{\cal R})\tau $
%      \end{algorithmic}
%    \label{alg:allocate_appendix}
%\end{algorithm}

\begin{lemma}\label{allocation}
Allocation $\pmb{\tau}$ in Algorithm \ref{alg:allocate} ensures that $(Z_k/\tau) \geq \frac{1}{ \sigma^*(B,{\cal R})}$ for all $k \in {\cal R}$.
\end{lemma}

 This is essentially the best allocation of the individual arm budgets in terms of ensuring good error bounds on the estimators $\hat{Y}_{k}^{\epsilon}$ for all $k \in \mathcal{R}$. Since \textbf{S1} is a special case of \textbf{S2}, to obtain the allocation for \textbf{S1} one needs to set the cost values $c_i$ set to $1$ for $i \in {\cal B}$ (\textit{difficult} arms) in the above formulation and $0$ otherwise.

\subsection{Putting it together: Online Analysis}
\label{sec:analysis}
We analyze Algorithm~\ref{alg:pickbest} phase by phase. With some abuse of notation, we redefine various quantities to be used in the analysis of the algorithm. Each quantity depends on the phase indices, as follows:
\begin{itemize}
\label{analysis_note}
\item $\mathcal{R}(l)$: Set of arms remaining after phase $l-1$ ends.
\item $\hat{Y}_{k}(l)$: The value of the estimator (in Algorithm~\ref{alg:pickbest}) for arm $k$ at the end of phase $l$.
\item $\hat{Y}_{H}(l)$: The value of the highest estimate $\max_k \hat{Y}_{k}(l) $ (in Algorithm~\ref{alg:pickbest}).
\item $ \mathcal{A}(l) \subseteq \mathcal{R}(l)$: Set of arms given by: 
\begin{align}
\label{eq:eliminate}
\mathcal{A}(l) := \left\{k \in \mathcal{R}(l) : \Delta_{k} > \frac{10}{2^l} \right\}.
\end{align}
\item $S_l$: Success event of phase $l$ defined as:
\begin{align}
\label{eq:success}
S_l := &\cap_{k \in \mathcal{R}(l), k \neq k^*} \left\{ \hat{Y}_k(l) \leq \mu_k + \frac{1}{2^{l-1}} \right\} \cap \left\{ \mu_k - \frac{3}{2^l} \leq \hat{Y}_{k^*}(l)\right\}.
\end{align}
\end{itemize}

Now we will establish that the occurrence of the event $S_{l}$ implies that all arms in $\mathcal{A}(l)$ gets eliminated at the end of phase $l$, while at the same time the optimal arm survives. Consider an arm $k \in \mathcal{A}(l)$. Given $S_l$ has happened we have:
\begin{align*}
\hat{Y}_{H}(l) &\geq \hat{Y}_{k^*}(l)  \geq \mu_{k^*} - \frac{3}{2^l} \\
\hat{Y}_{k}(l)  &\leq \mu_k + \frac{1}{2^{l-1}}
\end{align*}
This further implies that $\hat{Y}_{H}(l)  - \hat{Y}_{k}(l)  \geq \Delta_{k} - 5/2^{l} > 5/2^{l}$. Therefore all the arms in $\mathcal{A}(l)$ are eliminated given $S_l$. Following similar logic it is also possible to show that the optimal arm survives. If $\hat{Y}_H(l)  = Y_{k*}(l) $ then it survives certainly. Now, given $S_l$, only arms in $\mathcal{R}(l) \setminus \mathcal{A}(l)$ can be the ones with the highest means. Consider arms $k \in \mathcal{R}(l) \setminus \mathcal{A}(l)$. Again given $S_l$ we have:
\begin{align*}
\hat{Y}_{k^*}(l)  &\geq \mu_{k^*} - \frac{3}{2^l} \\
\hat{Y}_{k}(l)  &\leq \mu_k + \frac{1}{2^{l-1}}
\end{align*}
Therefore, we have $Y_{k}(l)  - Y_{k^*}(l)  \leq 5/2^{l} - \Delta_k < 5/2^{l}$. Therefore, the optimal arm survives. 

It would seem that now it would be easy to analyze the probability of the event $S_l$, using Theorem~\ref{thm:uniestimator}. However, the bound in Theorem~\ref{thm:uniestimator} depends on the sequence of arms eliminated so far in each phase. Therefore it is imperative to analyze $S_{1:l}$, that is the event that all phases from $1,2,..,l$ succeed. Let $B_l = \PP(S_l^c \vert S_{1:l-1})$. So, we have the chain:
\begin{align*}
\PP(S_{1:l}) &\geq \PP(S_{1:l} \vert S_{1:l-1})\PP(S_{1:l-1}) \\
&\geq  \PP(S_{1:l} \vert S_{1:l-1})\PP(S_{1:l-1} \vert S_{1:l-2})\PP(S_{1:l-2}) \\
&\geq \prod_{i = 0}^{l-2}\PP(S_{1:l-i} \vert S_{1:l-i-1})P(S_{1})\\
&= \prod_{s = 1}^{l}(1 - B_s) \\
&\geq 1 -\sum_{s = 1}^{l}B_s
\end{align*}
The advantage of analyzing the probability of $S_{1:l}$ is that given $S_{1:l}$ we know the exact sequences of the arms that have been eliminated till Phase $l$. This gives us exact control on the exponents in the bound of Theorem~\ref{thm:uniestimator}. Given $S_{1:s-1}$ we have,
\begin{align*}
\mathcal{R}(s) \subseteq \mathcal{R}^*(s) := \left\{ k : \Delta_{k} \leq \frac{10}{2^{s-1}} \right\}. 
\end{align*}
Recall that the budget for the samples of each arm in any phase $s$, is decide by solving the LP in Algorithm~\ref{alg:allocate}. Therefore, given $S_{1:s-1}$, we have $\sigma^*(B,{\cal R}(s)) \geq \sigma^*(B,{\cal R}^*(s))$. Therefore, we have the following key lemma. 
\begin{lemma}
\label{lem:negassociation}
We have:
\begin{align}
 B_l &:= \PP\left(S_{l}^c \vert S_{1:l-1} \right) \leq  2\lvert \mathcal{R}^*(l) \rvert \exp \left( -\frac{2^{-2(l-1)}\tau(l) v^*(B,{\cal R}^*(l))^2}{8l^2}\right)
\end{align}
\end{lemma}
\begin{proof}
Note that in this phase we set $\eta_{kj} = 2lM_{kj}$. Setting $\epsilon=2^{-(l-1)}$ and $\delta=2^{-(l-1)}$ in Theorem~\ref{thm:uniestimator} and by Lemma~\ref{allocation} we have:
\begin{align}\label{Phaselanalysis}
&\PP\left(\mu_k - \frac{3}{2^l} \leq \hat{Y}_k(l) \leq \mu_k + \frac{1}{2^{l-1}} \right) \geq 1 - 2\exp \left( -\frac{2^{-2(l-1)}\tau(l)v^*(B,{\cal R}^*(l))^2}{8l^2 }\right)
\end{align}
Note that the samples considered in phase $l$ are independent of the event $S_{1:l-1}$. Doing a union bound of the event complementary to the success event in (\ref{eq:success}), for all the remaining arms in $\mathcal{R}^*(\ell)$ implies the result in the Lemma.
\end{proof}

Now we are at a position to introduce our main results as Theorem~\ref{thm:detailed}.
\begin{theorem}
\label{thm:detailed}
Consider a problem instance with $K$ candidate arms $\{\mathrm{P}_{k}(V \vert pa(V) \}_{k=0}^{K-1}$. Let the gaps from the optimal arm be $\Delta_k$ for $k \in [K]$. Let us define the following important quantities:
\begin{align}
\label{rdelta}
{\cal R}^*(\Delta_k) = \left\{ s: \Bigg \lfloor \log_2 \left(\frac{10}{\Delta_s} \right) \Bigg \rfloor \geq \Bigg \lfloor  \log_2 \left(\frac{10}{\Delta_k }\right) \Bigg \rfloor \right\}
\end{align}
\begin{align}
\label{eq:HK}
\bar{H}_{k} = \max_{\{l : \Delta_l \geq \Delta_k\}} \frac{\log_2(10/\Delta_l)^3}{(\Delta_l/10)^2 v^*(B,{\cal R}^*(\Delta_l))^2}
\end{align}
  \begin{align}
  \label{Hstar}
  \bar{H} = \max_{k \neq k^*} \frac{\log_2(10/\Delta_k)^3}{(\Delta_k/10)^2 v^*(B,{\cal R}^*(\Delta_k))^2}
  \end{align}
 Algorithm~\ref{alg:pickbest} satisfies the following guarantees:
 
 {\noindent 1. } The simple regret is bounded as: 
 \begin{align*}
&r(T,B) \leq  2K^2\sum_{\substack{k \neq k^*: \\ \Delta_k \geq 10/\sqrt{T} }} \Delta_k \log_2\left(\frac{20}{\Delta_k}\right) \exp \left( - \frac{T}{ 2\bar{H}_{k} \overline{\log}(n(T))} \right) \\
& + \frac{10}{\sqrt{T}}\mathds{1} \left\{\exists k\neq k^* \text{ s.t } \Delta_k < 10/\sqrt{T} \right\} 
 \end{align*}
 
 {\noindent 2. } The error probability is bounded as: 
 \begin{align*}
  e(T,B) \leq 2K^2 \log_2 (20/\Delta) \exp \left( - \frac{T}{ 2\bar{H} \overline{\log}(n(T))} \right)
  \end{align*}
  The bound on the error probability only holds if $\Delta_k \geq 10/\sqrt{T}$ for all $k \neq k^*$.
\end{theorem}
\begin{proof}
Recall that the simple regret is given by:
\begin{align}
\label{regret}
r(T,B) = \sum_{k \neq k^*} \Delta_{k} \PP\left( \hat{k}(T,B) = k\right)
\end{align}
Let us introduce some further notation. Let us define the phase at which an arm is \textit{ideally} deleted as follows:
\begin{align}
\label{idealphase}
\gamma_{k} := \gamma(\Delta_k) := l \text{ if } \frac{10}{2^{l}} < \Delta_{k} \leq \frac{10}{2^{l-1}}
% * <rajat.sen@utexas.edu> 2016-11-17T21:57:57.183Z:
%
% ^.
\end{align}
Therefore we have the following chain:
\begin{align*}
&\PP\left( \hat{k}(T,B) = k\right)  \overset{a} \leq \PP\left(S^c_{1:\gamma_{k}} \right) \\
&\leq \sum_{l = 1}^{\gamma_k}B_l \\
&\leq \sum_{l = 1}^{\gamma_k}2\lvert \mathcal{R}^*(l) \rvert\exp \left( -\frac{2^{-2(l-1)}\tau(l)v^*(B,{\cal R}^*(l))^2}{8l^2}\right)
\end{align*}
provided $\Delta_k \geq 10/\sqrt{T}$. Justification for (a) - If arm $k$ is chosen finally, it implies that it is not eliminated at phase $\gamma_k$. Therefore the regret of the algorithm is given by:
\begin{align}
\label{regretbound}
r(T,B) &\leq \sum_{\{k \neq k^*: \Delta_k \geq 10/\sqrt{T} \}} \Delta_k \left(\sum_{l = 1}^{\gamma_k}2\lvert \mathcal{R}^*(l) \rvert \right. \left. \exp \left( -\frac{2^{-2(l-1)}\tau(l)v^*(B,{\cal R}^*(l))^2}{8l^2}\right)\right) \\ 
& + \frac{10}{\sqrt{T}}\mathds{1} \left\{\exists k\neq k^* \text{ s.t } \Delta_k < 10/\sqrt{T} \right\}
\end{align}
Let $\ell_1, \ell_2 ..\ell_s = \gamma_k$ such that $\mathcal{R}^*(\ell)$ changes value only at these phases. Let us set $\ell_{s+1}=\ell_s+1$ for  convenience in notation. Combining this notation with (\ref{regretbound}) we have:
\begin{align}
\label{eq:RB2}
r(T,B) &\leq \sum_{\{k \neq k^*: \Delta_k \geq 10/\sqrt{T} \}} \Delta_k \left(\sum_{i=1}^{s} 2|\mathcal{R}^*(\ell_i)| \left(\ell_{i+1}-\ell_{i}\right) \right. \left. \exp \left( - \frac{2^{-2\ell_i} T v^*(B,{\cal R}^*(l_i))^2}{2\ell_i^3 \overline{\log}(n(T))} \right)\right) \\ \nonumber
& + \frac{10}{\sqrt{T}}\mathds{1} \left\{\exists k\neq k^* \text{ s.t } \Delta_k < 10/\sqrt{T} \right\}
\end{align}
Consider the phase $\ell_i$ when ideally at least an arm leaves. Let one of those arms be $l$. Recall that, $\gamma_l$ is the phase where the arm ideally leaves according to (\ref{idealphase}). Therefore, $\gamma_l = \ell_i$. Also it is easy to observe that: ${\cal R}^*(\Delta_l) = {\cal R}^*(\ell_i)$. We have,
    \begin{equation}\label{Deltal}
        \ell_i \geq  \log_2 (10/\Delta_l)
    \end{equation}
as $\ell_{i+1} - \ell_i \leq \log_2 (20/\Delta_k), ~ i \leq s$. Further, for every $\ell_i < \gamma_k$, there is at least one distinct arm $l$ $: \gamma_l= \ell_i$. This is because an arm leaves only once ideally. Further, we associate $\ell_s$ with arm $k$ although other arms may leave at the phase $\ell_s=\gamma_k$. Further, all arms $l$ associated with $\ell_i < \gamma_k$ are such that $\Delta_l \geq \Delta_k$. This is because of (\ref{idealphase}) and the fact that $\ell_i < \ell_s=\gamma_k$.  Therefore, the r.h.s in Equation~(\ref{eq:RB2}) is upper bounded as follows:
\begin{align*}
 \label{eq:RB3}
r(T,B) &\leq \sum_{\{k \neq k^*: \Delta_k \geq 10/\sqrt{T} \}} \Delta_k \left(\sum_{\{l : \Delta_l \geq \Delta_k \}} 2|\mathcal{R}^*(\Delta_l)| \right.  \left. \log_2(20/\Delta_k) \exp \left( - \frac{ (\Delta_l/10)^2 T v^*(B,{\cal R}^*(\Delta_l))^2}{2\log_2(10/\Delta_l)^3 \overline{\log}(n(T))} \right)\right) \\ 
& + \frac{10}{\sqrt{T}}\mathds{1} \left\{\exists k\neq k^* \text{ s.t } \Delta_k < 10/\sqrt{T} \right\} \\
& \overset{(a)} \leq  \sum_{\{k \neq k^*: \Delta_k \geq 10/\sqrt{T} \}} \Delta_k  \left(\sum_{\{l : \Delta_l \geq \Delta_k \}} 2|\mathcal{R}^*(\Delta_l)| \right)    \log_2(20/\Delta_k) \exp \left( - \frac{T}{ 2\bar{H}_{k} \overline{\log}(n(T))} \right) \\
& + \frac{10}{\sqrt{T}}\mathds{1} \left\{\exists k\neq k^* \text{ s.t } \Delta_k < 10/\sqrt{T} \right\} \\
& \overset{(b)} \leq  2K^2\sum_{\{k \neq k^*: \Delta_k \geq 10/\sqrt{T} \}} \Delta_k  \log_2(20/\Delta_k) \exp \left( - \frac{T}{ 2\bar{H}_{k} \overline{\log}(n(T))} \right) \\
& + \frac{10}{\sqrt{T}}\mathds{1} \left\{\exists k\neq k^* \text{ s.t } \Delta_k < 10/\sqrt{T} \right\} 
\end{align*}
Here (a) is by definition of $\bar{H}_{k}$ while (b) is because $\lvert \mathcal{R}^*(\Delta_l) \rvert \leq K$ and there are at most $K$ terms in the summation.

Another quantity of interest here is the error probability. We will only provide bounds on the error probability $e(T,B)$ when we have $\Delta_k > 10/\sqrt{T}$ for all $k \neq k^*$. Let $\Delta = \min_{k \neq k^*} \Delta_k$ and $\gamma^* = \gamma(\Delta)$. In this case we have:
\begin{align}
e(T,B) &\leq 1 - \PP\left( S_{1:\gamma^*}\right) \nonumber \\
&\leq \sum_{l = 1}^{\gamma^*} B_l \\
& \leq\sum_{l = 1}^{\gamma^*}2\lvert \mathcal{R}^*(l) \rvert\exp \left( -\frac{2^{-2(l-1)}\tau(l)v^*(B,{\cal R}^*(l))^2}{8l^2}\right) \nonumber \\
%&= O \left(K\log\left(\frac{1}{\Delta}  \right) \exp \left( -\frac{(\Delta/10)^2T(\gamma^*)Z^*(\gamma^*)^{2}}{4(\gamma^*)^2T(\gamma^*)^{2} }\right) \right)
& \overset{a}\leq\sum_{l = 1}^{\gamma^*}2\lvert \mathcal{R}^*(l) \rvert\exp \left( -\frac{2^{-2l}T v^*(B,{\cal R}^*(l))^2}{2l^3 \overline{\log}(n(T))}\right) \label{errorprob}
\end{align}
(a)- This follows from the definition of $\tau(l)$.

As before let $\ell_1=1, \ell_2 ..\ell_m \leq \gamma^*$ such that $R^*(\ell)$ changes value only at these phases. Let us set $\ell_{m+1}=\ell_m+1$ for  convenience in notation. 
Then, $e(T,B)$ in (\ref{errorprob}) is upper bounded by:
\begin{align}\label{newerror}
  e(T,B) &\leq \sum_{i=1}^{m} 2|\mathcal{R}^*(\ell_i)| \left(\ell_{i+1}-\ell_{i}\right) \exp \left( - \frac{2^{-2\ell_i} T v^*(B,{\cal R}^*(l_i))^2}{2\ell_i^3 \overline{\log}(n(T))} \right)
\end{align}

 Consider the phase $\ell_i$ when ideally at least an arm leaves. Let one of those arms be $k$. Recall that, $\gamma_k$ is the phase where the arm ideally leaves according to (\ref{idealphase}). Therefore, $\gamma_k = \ell_i$. Also it is easy to observe that: ${\cal R}^*(\Delta_k) = {\cal R}^*(\ell_i)$.  Then,
    \begin{equation}\label{Deltak}
        \ell_i \geq  \log_2 (10/\Delta_k)
    \end{equation}
  Further, for every $\ell_i$, there is a distinct and different $k: \gamma_k= \ell_i$. This is because an arm leaves only once ideally. Therefore, according to (\ref{Deltak}) and (\ref{newerror}) we have:
 
 \begin{align}
  e(T,B) &\leq \sum_{k \neq k^*} 2|\mathcal{R}^*(\Delta_k)| \gamma^* \exp \left( - \frac{ (\Delta_k/10)^2 T v^*(B,{\cal R}^*(\Delta_k))^2}{2\log_2(10/\Delta_k)^3 \overline{\log}(n(T))} \right) \nonumber \\
       &\leq 2K^2 \log_2 (20/\Delta)  \exp \left( - \frac{T}{ 2\bar{H} \overline{\log}(n(T))} \right)  \label{eq:llerror}
  \end{align}
  Here, we have used the definition of $\bar{H}$ and the fact that $|\mathcal{R}^*(\Delta_k)| \leq K$ and $\gamma^* \leq \log (20/\Delta) $
\end{proof}

\section{More Experiments}
\label{sec:moresims}
\FloatBarrier
In this section we provide more details about our experiments.

\subsection{More on Flow Cytometry Experiments}
\label{sec:morecyto}
In this section we give further details on the flow cytometry experiments. As detailed in the main paper we use the causal graph in Fig. 5(c) in~\cite{mooij2013cyclic} (shown in Fig.~\ref{fig:cgraph1}) as the ground truth. Then we fit a GLM gamma model~\cite{hardin2007generalized} between each node in the graph and its parents using the observational data-set. The GLM model produces a highly accurate representation of the flow cytometry data-set. In Fig.~\ref{fig:histogram} we plot the histogram for the activation of an internal node \textit{pip2} from the real data and samples generated from the GLM probabilistic model. It can be seen that the histograms are very close to each other.

%\begin{figure}
%	\centering
%	\includegraphics[width=0.8\linewidth]{./figs/pip2.pdf}
%	\caption{\small Histograms of data from the cytometry data-set and from the GLM trained for the activations of an internal node \textit{pip2}. \normalsize} \label{fig:histogram}
%	%\setlength{\belowcaptionskip}{-30pt}
%\end{figure}

In Fig.~\ref{fig:divM} we plot the performance of SRISv2 when the divergence metric is replaced by $\mathrm{KL}$-divergence. In one of the plots SRISv2 is modified by setting $M_{ij} = 1 + \mathrm{KL}(P_i || P_j)$. It can be seen that the performance degrades, which signifies that our divergence metric is fundamental to the problem.

%\begin{figure}
%	\centering
%	\includegraphics[width=0.8\linewidth]{./figs/low_divMM.pdf}
%	\caption{\small Comparison of SRISv2 with two different divergence metric. It shows that our divergence measure is fundamental for good performance. The experiments have been performed in a setting identical to Fig.~\ref{fig:cgraph3}\normalsize} \label{fig:divM}
%	%\setlength{\belowcaptionskip}{-30pt}
%\end{figure}
\begin{figure}
	\centering
	\subfloat[][\small Histograms of data from the cytometry data-set and from the GLM trained for the activations of an internal node \textit{pip2}. \normalsize]{\includegraphics[width = 0.45\linewidth]{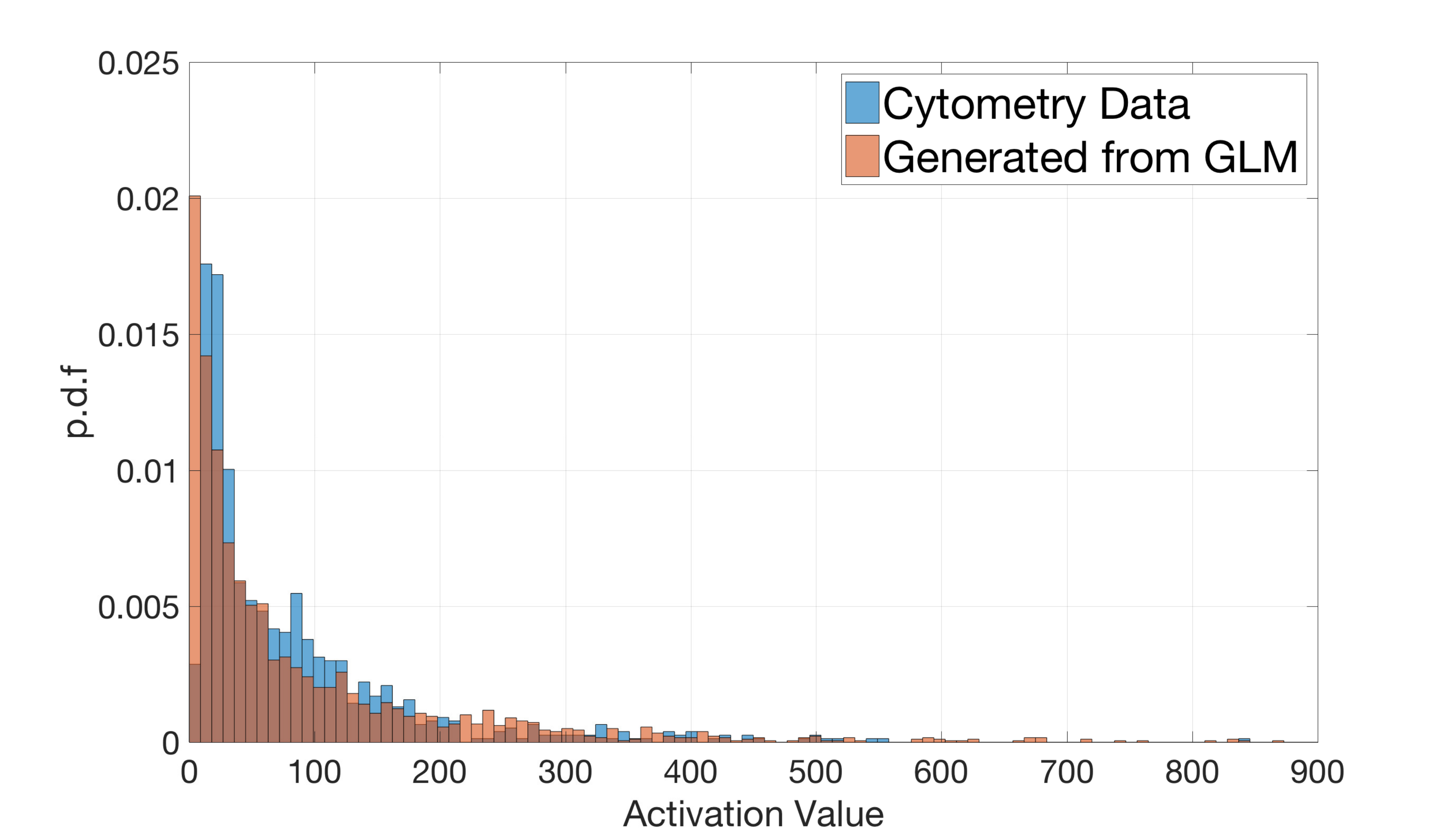}\label{fig:histogram}} \hfill
	\subfloat[][\small Comparison of SRISv2 with two different divergence metric. It shows that our divergence measure is fundamental for good performance. The experiments have been performed in a setting identical to Fig.~\ref{fig:cgraph3}\normalsize]{\includegraphics[width = 0.45\linewidth]{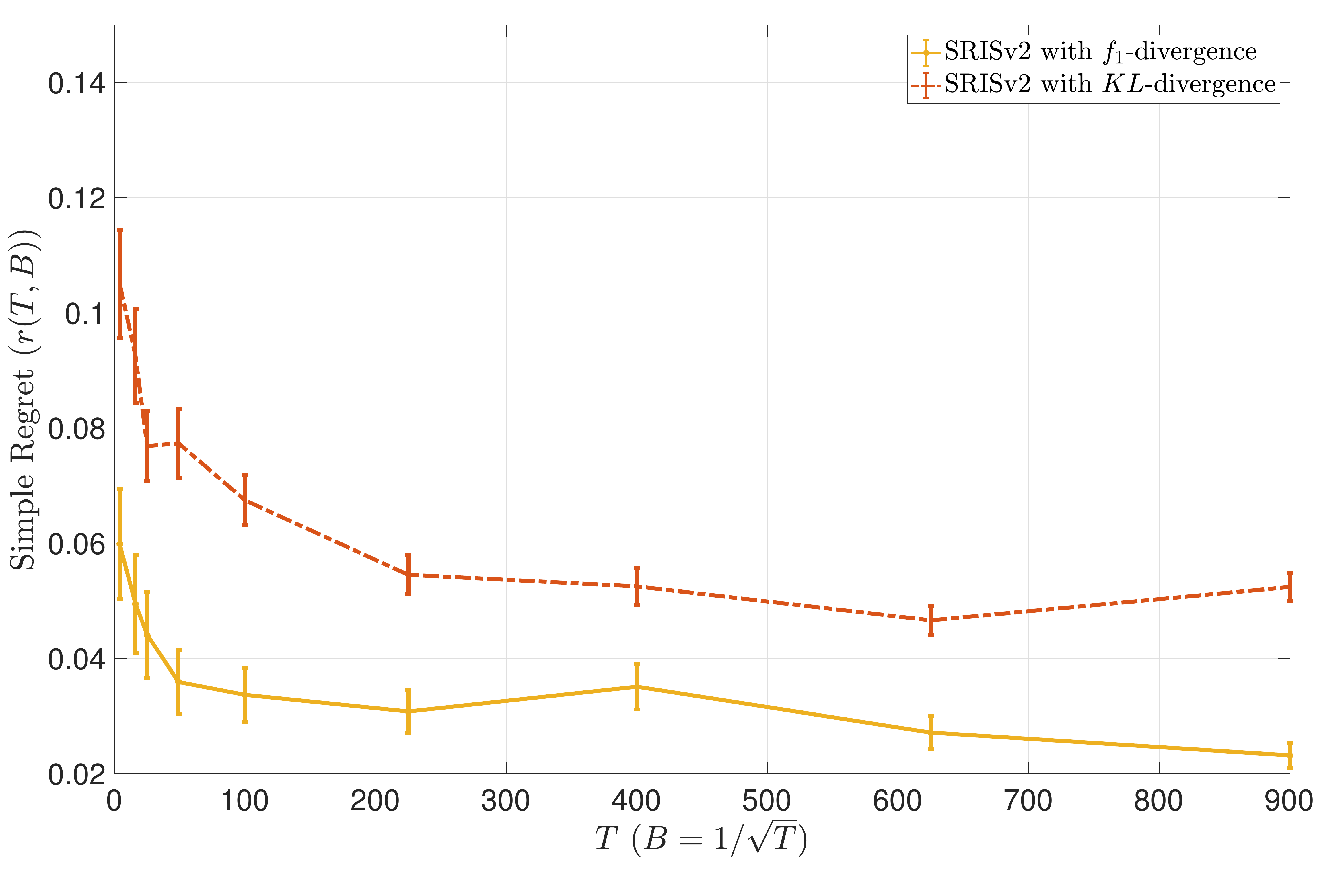}\label{fig:divM}}
	\caption{Illustration of our experimental methods.}
	\label{fig:all_comp}
\end{figure}

\subsection{More on Interpretation of Inception Deep Network}
\label{sec:moreincep}

In this section we describe the methodology of our model interpretation technique in more detail. In Section~\ref{sec:inception} we have described how the best arm algorithm can be used to pick a distribution over the superpixels of an image, that has the maximum likelihood of producing a certain classification from Inception. Here, we describe how the distributions over the superpixels are generated and how they are used subsequently. The arm distributions are essentially points in the $n$-dimensional simplex (where $n$ is the number of superpixels into which the image is segmented). These distributions are generated in a randomized fashion using the following methods:
\begin{enumerate}
	\item Generate a point uniformly at random in the $n$-dimensional simplex.
	\item Randomly choose $l < n$ superpixels. Make the distribution uniform over them and $0$ elsewhere.
	\item Randomly choose $l < n$ superpixels. The probability distribution is a uniformly chosen random point over the $l$-dimensional simplex with support on the $l$ chosen superpixels and $0$ elsewhere.
	\item Start a random walk from a few superpixels which traverses to adjacent superpixels at each step. Stop the random-walk after a certain number of steps and choose the superpixels touched by the random walk. Then choose a uniform distribution over the super pixel support or choose a random distribution from the simplex of probability distributions over the support of the chosen superpixels (like in the previous point) and $0$ elsewhere. This method uses the geometry of the image.
\end{enumerate}
Note that all the above methods do not depend on the specific content of the images. Using the above methods $L$ \textit{pull} arms are chosen which are used to collect the rewards. Further there are $K$ \textit{opt} arms that are the interventions to be optimized over. When an arm is used to sample, then $m \ll n$ superpixels are chosen with replacement from the distribution of that arm. These pixels are preserved in the original image while everything else is blurred out before feeding this into the neural network. Thus if the distribution corresponding to an arm is $P$, then the actual distribution to be used for the importance sampling is the product distribution $P^{m}$. Note that the \textit{pull} arms are separate from the \textit{opt} arms. The true counterfactual power of our algorithm is showcased in this experiment, as we are able to optimize over a large number of interventions that are never physically performed.

\end{document}